\documentclass{article} 
\usepackage{iclr2026_conference,times}

\usepackage{enumitem}
\usepackage{caption}
\usepackage{fancybox}

\usepackage{algpseudocode}

\usepackage[utf8]{inputenc} 
\usepackage[T1]{fontenc}    
\usepackage{hyperref}       
\usepackage{url}            
\usepackage{booktabs}       
\usepackage{amsfonts}       
\usepackage{nicefrac}       
\usepackage{microtype}      
\usepackage{xcolor}         
\definecolor{citecolor}{HTML}{0071BC}
\definecolor{linkcolor}{HTML}{D32F2F}
\definecolor{cellcolor}{HTML}{E3F2FD}
\definecolor{red}{HTML}{D32F2F}
\definecolor{magenta}{HTML}{D81B60}

\usepackage{amsmath}
\usepackage{amssymb}
\usepackage{mathtools}
\usepackage{amsthm}

\usepackage[capitalize,noabbrev]{cleveref}

\theoremstyle{plain}
\newtheorem{theorem}{Theorem}[section]
\newtheorem{proposition}[theorem]{Proposition}

\theoremstyle{definition}
\newtheorem{definition}[theorem]{Definition}

\theoremstyle{remark}
\newtheorem{remark}[theorem]{Remark}

\usepackage[textsize=tiny]{todonotes}
\usepackage{microtype}
\usepackage{graphicx}

\usepackage{subcaption}
\usepackage{booktabs} 

\usepackage{pgfplots}
\pgfplotsset{compat = newest}
\usepackage{multirow}

\usepackage{framed}

\usepackage{algorithm}
\usepackage{mathrsfs}
\usepackage{mathtools}
\usepackage{listings}
\usepackage{makecell}
\usepackage{colortbl}
\usepackage{color}
\usepackage{wrapfig}
\usepackage{cancel}
\usepackage{soul,xcolor}
\usepackage{pifont}
\usepackage{tikz}
\usetikzlibrary{shapes, positioning, arrows.meta, calc, decorations.pathmorphing, quotes}
\usepackage{todonotes}
\renewcommand{\cite}{\citep}

\usepackage[many]{tcolorbox}

\newtcolorbox{resultbox}[1][]{
    enhanced, 
    breakable, 
    colback=gray!5!white, 
    colframe=black!70!white, 
    boxrule=0.5pt, 
    arc=3pt, 
    left=6pt, right=6pt, top=6pt, bottom=6pt, 
    fonttitle=\bfseries, 
    coltitle=black, 
    attach boxed title to top center={yshift=-1mm}, 
    boxed title style={colback=white, colframe=black!70!white, boxrule=0.5pt}, 
    title=#1 
}

\usepackage{amsmath,amsfonts,bm}

\def\eqref#1{equation~\ref{#1}}

\def\1{\bm{1}}

\DeclareMathAlphabet{\mathsfit}{\encodingdefault}{\sfdefault}{m}{sl}
\SetMathAlphabet{\mathsfit}{bold}{\encodingdefault}{\sfdefault}{bx}{n}

\usepackage{hyperref}
\usepackage{url}

\title{From Intent to Evidence: A Categorical Approach for Structural Evaluation of Deep Research Agents}

\iclrfinalcopy

\author{
  Shuoling Liu\textsuperscript{1,2}, Zhiquan Tan\textsuperscript{2}, Kun Yi\textsuperscript{2}, Hui Wu\textsuperscript{2}, Yihan Li\textsuperscript{2}, Jiangpeng Yan\textsuperscript{2},\\ \textbf{ Liyuan Chen\textsuperscript{2}, Kai Chen\textsuperscript{1}, Qiang Yang\textsuperscript{3}}\\
  1. The Hong Kong University of Science and Technology\\
  2. E Fund Management Co., Ltd.\\
  3. The Hong Kong Polytechnic University\\
}

\begin{document}

\maketitle

\begin{abstract}
Deep Research Agents (DRAs) aim to answer complex questions by searching the web, checking evidence, and synthesizing conclusions across heterogeneous sources. We introduce a category-theoretic framework for evaluating and improving such agents. The framework treats deep research as a structured mapping from user intent to evidence-grounded conclusions, making retrieval traces, cross-source alignment, and final synthesis explicit. Guided by this view, we derive a mechanism-aware benchmark of 296 bilingual questions. The benchmark targets four structural skills central to real research: following multi-hop evidence chains, verifying claims across sources, re-ordering fragmented information, and rejecting unsupported assumptions. We evaluate 16 frontier systems with human verification and find that these structural tasks remain highly challenging: the best system reaches only 19.9\% average accuracy. The results show that strong agents can sometimes reorganize evidence and detect false premises, but still struggle with long-horizon synthesis and intersection-heavy verification. Beyond evaluation, the same theory also leads to practical system improvements. We instantiate theory-guided interventions such as tracked search, which preserves retrieval traces, and category tools, which add explicit verification and synthesis steps. These interventions yield measurable gains in API-based deep research systems. Our work therefore provides both a challenging benchmark and concrete design guidance for building more reliable research agents.\footnote{Our implementation will be available at \url{https://github.com/tzq1999/CDR}.}
\end{abstract}

\section{Introduction}

The recent leap in large language models (LLMs) has sparked the rapid development of autonomous agents, with Deep Research Agents (DRAs) emerging as a practical solution for complex tasks \citep{wu2025webwalkerbenchmarkingllmsweb}. Unlike early systems confined to static knowledge recall \citep{wei2025browsecomp,zhou2025browsecompzh,chen2025browsecompplus}, true deep research is inherently dynamic \citep{nie2024survey-financial}. It requires an agent to actively seek out, cross-validate, and stitch together information from diverse and frequently conflicting sources \citep{hu2025finsearchcomprealisticexpertlevelevaluation}. To provide genuine decision support, a DRA must go beyond simple retrieval; it needs to master the art of filtering noise, resolving contradictions, and synthesizing fragmented evidence into well-grounded judgments \citep{du2025deepresearchbenchcomprehensivebenchmark}.

However, a rigorous theoretical modeling of deep research agent behavior remains conspicuously absent from the literature. Without a formal mathematical foundation, current evaluation paradigms rely heavily on purely empirical and ad-hoc designs \citep{xu2025comprehensivesurveydeepresearch}. These heuristic benchmarks lack the systematism required to explicitly disentangle and measure an agent's concrete capabilities, often blurring the lines between genuine reasoning, basic retrieval, and simple memorization \citep{shi2025deepresearchsystematicsurvey}. Compounding this issue is the limited complexity of the evaluation tasks themselves; the queries designed in existing datasets rarely pose the level of difficulty needed to genuinely stress-test an agent's capacity for long-horizon synthesis and ambiguity resolution. This highlights a critical gap: the community lacks a theoretically grounded benchmark that provides both systematic capability verification and appropriately challenging tasks. Addressing this requires a fundamental shift from fragmented empirical testing toward a formal framework capable of rigorously modeling, and subsequently evaluating, the agent's entire research ability.

\begin{figure}[htbp]
    \centering
    \includegraphics[width=0.8\textwidth]{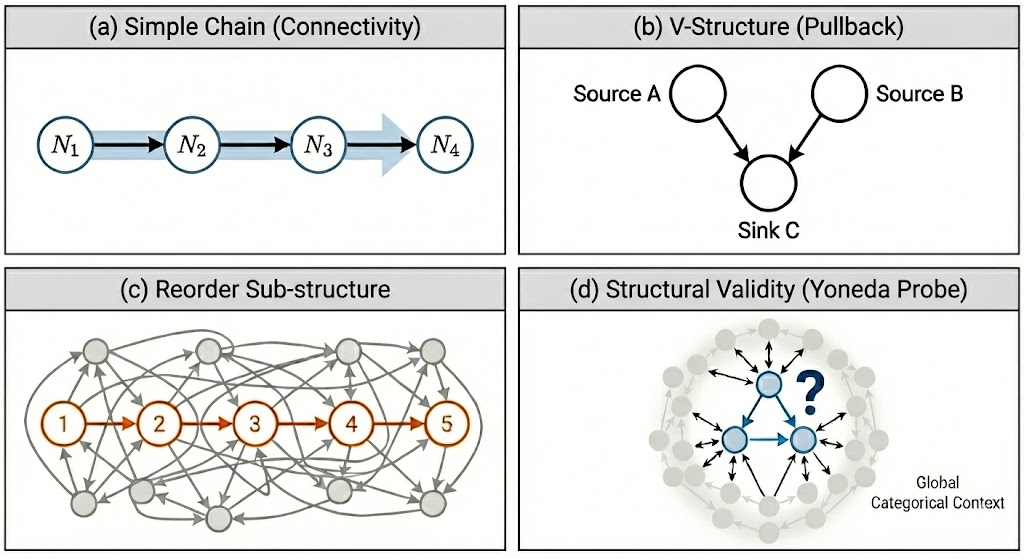}
    
    \caption{An illustrative picture of 4 categorical evaluation, definitions can be found in Section \ref{prelim}.}
    \label{fig:evalframe} 
\end{figure}

We formalize the behavior of the Deep Research Agent through the lens of category theory \citep{mac2013categories}, since the deep-research workflow is fundamentally a structural mapping: an abstract intent is mediated through a noisy external web and transformed into a logically constrained conclusion space. Rather than treating the pipeline as an ad-hoc collection of heuristics, this perspective models the workflow as a composition of structure-preserving maps (functors), which enables precise, testable statements about agent behaviours. Grounded in this categorical framework, we systematically construct a challenging benchmark with 296 bilingual questions that adversarially stress-tests agents along four interpretable axes: traversing sequential connectivity chains, verifying intersections within V-structure pullbacks, imposing topological ordering on retrieved substructures, and performing ontological falsification via the Yoneda Probe—so as to provide diagnostic, mechanism-aware evaluations. See Figure \ref{fig:evalframe} for an illustrative view.

We systematically benchmark $16$ frontier systems spanning reasoning models, search-augmented systems, closed-source deep research products, and our own theory-guided API-based variants using a strict, human-verified evaluation pipeline. The results establish a persistently low baseline---with the best system reaching only $19.9\%$ average performance on our benchmark---rigorously exposing the difficulty of formal structural stress-testing. Our evaluation reveals a sharp capability split. On one hand, advanced deep research pipelines demonstrate emerging strengths: they can successfully redefine dynamic topological re-ordering and exhibit robust ontological verification, sometimes matching pure reasoning models in their ability to falsify hallucinated premises. On the other hand, critical vulnerabilities remain unresolved. Systems almost universally collapse on multi-hop structural synthesis and exhibit absolute blind spots under specific mathematical constraints. Crucially, the large variance both across systems and across task families indicates a lingering reliance on brittle heuristics rather than a unified structural understanding.

Importantly, these findings are not only diagnostic. Because our framework factorizes deep research into retrieval structure and evidence-grounded synthesis, it also suggests concrete system interventions. We instantiate these ideas as two practical modules. First, \textit{tracked search} turns otherwise latent browsing jumps into explicit, auditable retrieval-graph edges, preserving the provenance needed by the Search Functor and making retrieval composition observable rather than implicit. Second, \textit{category tools} implement pullback-style verification and colimit-style synthesis, so that cross-source claims are merged only when they admit evidentiary pre-images and provenance-preserving aggregation, directly targeting failures of the Reasoning Functor. These are not merely conceptual embellishments: in controlled API-based systems, adding category tools on top of tracked search improves GPT-5 from $45.68\%$ to $47.53\%$ on DeepResearchBench and from $10.48\%$ to $13.29\%$ on our benchmark, with similar gains for Grok-4.1-fast. The theory therefore functions not only as an evaluation language, but also as a concrete design blueprint for stronger deep research agents.

This paper makes three contributions. First, we provide a category-theoretic formulation of deep research in terms of the Search Functor, the Reasoning Functor, pullbacks, colimits, and Yoneda-style falsification, thereby turning informal workflow intuitions into precise structural claims. Second, we instantiate this theory as a challenging benchmark and a human-verified evaluation protocol that expose concrete capability bottlenecks beyond broad report quality. Third, we translate the theory into concrete design modules: tracked search preserves retrieval traces, while categorical verification and synthesis operators improve evidence alignment and aggregation, yielding measurable gains in API-based implementations and in some cases narrowing the gap to closed-source deep research products. Together, these results underscore that while top-tier autonomous agents can increasingly unify search and reasoning, generalized mastery over complex structural topologies remains a formidable open challenge.

\section{Preliminary} \label{prelim}

To formalize the behavior of deep research agents, we briefly introduce several foundational concepts from category theory \citep{mac2013categories}. Our presentation balances mathematical precision with intuitive interpretations tailored for machine learning.

\begin{definition}[Category]
A category $\mathcal{C}$ consists of a class of objects $\mathrm{Ob}(\mathcal{C})$ and, for any pair of objects $X, Y \in \mathrm{Ob}(\mathcal{C})$, a set of morphisms $\mathrm{Hom}_{\mathcal{C}}(X,Y)$. These morphisms are equipped with a composition law: for $f \in \mathrm{Hom}_{\mathcal{C}}(X,Y)$ and $g \in \mathrm{Hom}_{\mathcal{C}}(Y,Z)$, their composition $g \circ f \in \mathrm{Hom}_{\mathcal{C}}(X,Z)$ is associative. Furthermore, every object $X$ has an identity morphism $\mathrm{id}_X \in \mathrm{Hom}_{\mathcal{C}}(X,X)$.
\end{definition}

\text{Intuition in ML:} A category abstracts a system of objects and their structure-preserving transformations. In our context, objects naturally correspond to representation spaces, datasets, or hypothesis classes, while morphisms represent the functions, network layers, or reasoning paths that map between them.

\begin{definition}[Functor]
Given categories $\mathcal{C}$ and $\mathcal{D}$, a functor $F: \mathcal{C} \to \mathcal{D}$ is a structure-preserving mapping that assigns each object $X \in \mathcal{C}$ to an object $F(X) \in \mathcal{D}$, and each morphism $f: X \to Y$ in $\mathcal{C}$ to a morphism $F(f): F(X) \to F(Y)$ in $\mathcal{D}$. A functor must preserve composition, i.e., $F(g \circ f) = F(g) \circ F(f)$, and identity, i.e., $F(\mathrm{id}_X) = \mathrm{id}_{F(X)}$.
\end{definition}

\text{Intuition in ML:} Functors provide a principled way to translate between different architectural paradigms or representation spaces. They ensure that the relational structure (the morphisms) of the original space is perfectly preserved during the mapping—analogous to how parameterization or embedding functions map discrete symbolic spaces into continuous latent spaces while preserving semantic relationships.

\begin{definition}[Pullback]
Given an object $Z$ and two morphisms $f: X \to Z$ and $g: Y \to Z$, their pullback is an object $P$ equipped with projection morphisms $\pi_1: P \to X$ and $\pi_2: P \to Y$ such that $f \circ \pi_1 = g \circ \pi_2$. Furthermore, $P$ satisfies a universal property: for any object $Q$ with morphisms $u: Q \to X$ and $v: Q \to Y$ satisfying $f \circ u = g \circ v$, there exists a unique morphism $h: Q \to P$ such that $\pi_1 \circ h = u$ and $\pi_2 \circ h = v$.
\end{definition}

\text{Intuition in ML:} The pullback computes the most general object whose images in $X$ and $Y$ are consistent under $f$ and $g$. In the category of sets, it is given by $P = \{ (x,y) \in X \times Y \mid f(x) = g(y) \}$. In machine learning, computing a pullback corresponds to enforcing equality constraints or performing multimodal fusion, where representations from different modalities ($X$ and $Y$) must align in a shared abstract semantic space ($Z$).

\begin{definition}[Limits and Colimits]
Let $J$ be an index category and $D: J \to \mathcal{C}$ be a diagram (a functor) in $\mathcal{C}$. The \emph{limit} of $D$, denoted $\lim D$, is an object $L \in \mathcal{C}$ equipped with projection morphisms $\pi_j: L \to D(j)$ that are compatible with all morphisms in $D$, satisfying a universal property analogous to the pullback. Dually, the \emph{colimit} of $D$, denoted $\operatorname{colim} D$, is the limit of the corresponding diagram in the opposite category $\mathcal{C}^{\mathrm{op}}$ (where all morphism directions are reversed).
\end{definition}

\text{Intuition in ML:} Limits and colimits generalize the notions of aggregation and composition. A limit (which encompasses products and pullbacks) aggregates information by finding the maximal jointly consistent representation that satisfies multiple structural constraints across the diagram. Conversely, a colimit (which encompasses disjoint unions and pushouts) formalizes the process of "gluing" or ensembling objects together—such as merging distinct state spaces or representations along a shared subspace.

\section{Categorical Formulation of Deep Research Agents}
\label{sec:formulation}

We formalize the end-to-end workflow of a Deep Research Agent (DRA) as a structural mapping from an abstract problem space to a structured conclusion space, fundamentally mediated by an external knowledge environment (the Web). While the empirical behaviors of large language models (LLMs) are often described using heuristic terminology, casting this transition into the language of category theory allows us to strictly define the requisite capabilities of an ideal DRA—such as task decomposition, information retrieval, and logical synthesis—through standard categorical constructs. More importantly, it provides a rigorous mathematical vocabulary to diagnose common agentic failure modes. More detailed proofs can be found in Appendix \ref{proof}. An illustrative figure can be found in Figure \ref{fig:overview}. 

\begin{figure}[htbp]
    \centering
    \includegraphics[width=1\textwidth]{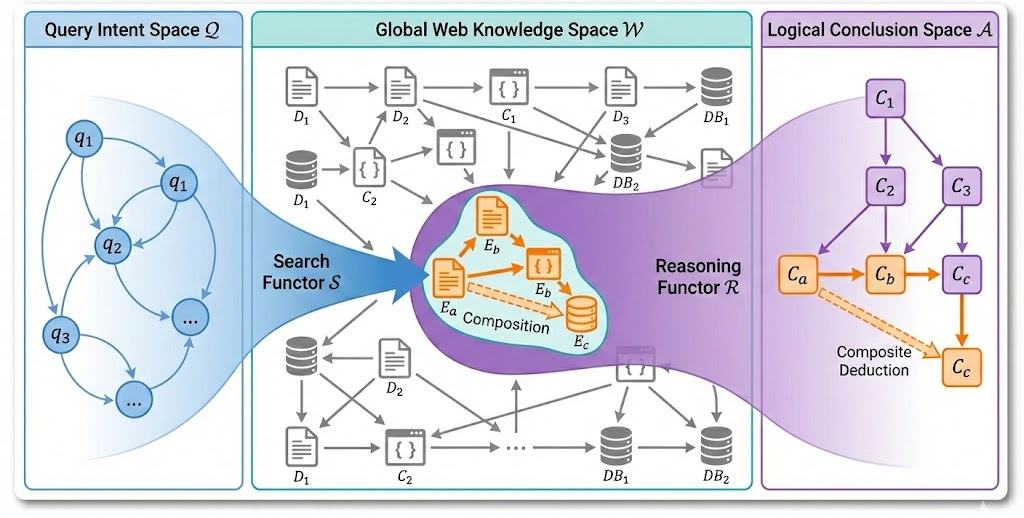}
    
    \caption{An illustrative picture of categorical view of the deep research agent workflow.}
    \label{fig:overview} 
\end{figure}

\subsection{Agent State Spaces as Categories}
Before defining the functors that govern the agent's dynamic behavior, we first establish the foundational categories that represent the distinct state spaces of the research process. For each conceptual phase, we specify its objects ($\mathrm{Ob}$) and morphisms ($\mathrm{Hom}$) to seamlessly bridge formal mathematics with intuitive AI operations.

\textbf{The Intent Category $\mathcal{Q}$ (Query Space).} 
The starting point of any research task resides in the Intent Space, which encapsulates the informational goals of the agent. Here, the objects $q \in \mathrm{Ob}(\mathcal{Q})$ represent all possible user queries, sub-problems, or atomic intent units. The morphisms $\mathrm{Hom}_{\mathcal{Q}}(q_1, q_2)$ represent \emph{task dependency} or \emph{refinement}. Specifically, a directed morphism $q_1 \xrightarrow{\text{dep}} q_2$ exists if resolving the sub-problem $q_2$ strictly depends on the resolution of $q_1$. In the context of autonomous agents, this category effectively models and formalizes the agent's internal task decomposition and planning graph.

\textbf{The Knowledge Category $\mathcal{W}$ (Web Space).} 
To execute these intents, the agent operates within the Knowledge Space, representing the vast, interconnected external environment. Its objects $w \in \mathrm{Ob}(\mathcal{W})$ are informational entities on the World Wide Web, such as web pages, raw documents, code snippets, or database records. The morphisms $\mathrm{Hom}_{\mathcal{W}}(w_1, w_2)$ denote \emph{structural or evidentiary linkages}. A morphism $w_1 \xrightarrow{\text{ref}} w_2$ exists if $w_1$ contains a hyperlink to $w_2$, or if the semantic content of $w_1$ objectively acts as a supporting premise or causal relation for $w_2$.

\textbf{The Retrieval Subcategory $\mathbf{Sub}(\mathcal{W})$ (Retrieved Context).} 
Because LLMs operate with finite context windows, the agent cannot process the entire Web simultaneously; instead, it interacts with a localized Retrieval Subspace. This is defined as a category formed by subgraphs (subcategories) of $\mathcal{W}$. Its objects $G \in \mathrm{Ob}(\mathbf{Sub}(\mathcal{W}))$ are the retrieved knowledge subgraphs $G \subseteq \mathcal{W}$ collected by the agent at any given step. The morphisms $\mathrm{Hom}_{\mathbf{Sub}}(G_1, G_2)$ are \emph{inclusion mappings} ($G_1 \hookrightarrow G_2$) or morphism restrictions. Intuitively, these mappings formalize the monotonic accumulation of gathered evidence as the agent systematically explores the Web.

\textbf{The Reasoning Category $\mathcal{A}$ (Logic/Answer Space).} 
Finally, the synthesized outputs reside in the Reasoning Category, which captures the internal cognitive workspace where final reports are synthesized. The objects $a \in \mathrm{Ob}(\mathcal{A})$ are formulated propositions, intermediate arguments, or final factual statements. The morphisms $\mathrm{Hom}_{\mathcal{A}}(a_1, a_2)$ represent strict \emph{logical entailment}. A morphism $a_1 \xrightarrow{\text{implies}} a_2$ exists if proposition $a_1$ is a sufficient condition, or provides robust deductive backing, for proposition $a_2$ ($a_1 \implies a_2$).

\begin{theorem}[Categorical Validity of Agentic State Spaces]\label{thm:categorical_validity}
    Let $\mathcal{Q}$, $\mathcal{W}$, $\mathbf{Sub}(\mathcal{W})$, and $\mathcal{A}$ be the agentic state spaces equipped with their respective objects and relations. If we define the morphisms in $\mathcal{Q}$, $\mathcal{W}$, and $\mathcal{A}$ as above, then all four spaces strictly satisfy the Eilenberg-Mac Lane axioms and thus constitute mathematically valid categories.
    \end{theorem}

Equipped with these formal state spaces, the overarching deep research process can now be elegantly modeled as a sequence of structure-preserving functors bridging $\mathcal{Q}$, $\mathbf{Sub}(\mathcal{W})$, and $\mathcal{A}$, which will concretely define the agent's core capabilities of searching and reasoning.

\subsection{Agent Workflows as Functors}

\paragraph{Search Functor}
We formalize the (ideal) search operation as a functor $S: \mathcal{Q} \to \mathbf{Sub}(\mathcal{W})$ that maps an abstract “problem flow” to a concrete “evidence flow.” An ideal search functor must satisfy two central categorical conditions—preservation of morphisms and preservation of composition—while the remaining conditions (object mapping and identity preservation) are typically straightforward to satisfy.

\begin{definition}[Morphism Mapping and Dependency Preservation] \label{def:search_morphism}
For any query dependency $f: q_1 \to q_2$ in $\mathcal{Q}$, the Search Functor must map it to a valid morphism $S(f): S(q_1) \to S(q_2)$ within the retrieval subspace. 
\end{definition}

\begin{remark}[Semantic interpretation of search dependency]\label{rem:search-dependency}
    Definition \ref{def:search_morphism} constrains the agent's search trajectory. If a question \(q_2\) logically depends on \(q_1\), then the retrieved material subobject \(S(q_2)\in\mathbf{Sub}(\mathcal{W})\) should extend \(S(q_1)\) in the categorical sense: there exists a monomorphism
    $\iota: S(q_1)\hookrightarrow S(q_2)$
    in \(\mathbf{Sub}(\mathcal{W})\). (When \(\mathbf{Sub}(\mathcal{W})\) is realised as a poset of concrete subsets, this condition reduces to \(S(q_1)\subseteq S(q_2)\).) 
    
    In practice, this means \(S(q_2)\) ought to organically extend \(S(q_1)\)—for example, by including pages reachable via hyperlinks appearing in \(S(q_1)\), or by containing temporal successors in a time-series search. Accordingly, when an agent decomposes a problem into subquestions, it should preserve and make available the entire set of previously accessed material (i.e., the access history/provenance) associated with those subquestions.
    \end{remark}

\begin{proposition}[Chain-of-Search Consistency] \label{prop:search_consistency}
An ideal deep research agent preserves compositionality during multi-hop retrieval (deep search). For a sequence of sub-task dependencies $q_1 \xrightarrow{f} q_2 \xrightarrow{g} q_3$, it strictly holds that $S(g \circ f) = S(g) \circ S(f)$.
\end{proposition}

\paragraph{The Reasoning Functor}

The logical synthesis process is formally defined as an ideal functor $R: \mathcal{W} \to \mathcal{A}$. It represents the fundamental capability of the agent to map unstructured entities from the global web knowledge space $\mathcal{W}$ into structured propositions in the abstract logic space $\mathcal{A}$. (In practice, during a specific task, this functor acts upon the retrieved evidence subgraph $\mathrm{Im}(S)$).

\begin{definition}[Morphism Mapping and Causality Preservation]\label{def:reasoning_morphism}
For any evidentiary morphism $f \in \mathrm{Hom}_{\mathcal{W}}(w_1, w_2)$—such as a hyperlink or an objective causal premise—the reasoning functor must produce a valid logical implication $R(f) \in \mathrm{Hom}_{\mathcal{A}}(R(w_1), R(w_2))$. Consequently, for any such pair $(w_1, w_2)$, the functor induces a mapping between hom-sets: $R_{w_1, w_2}: \mathrm{Hom}_{\mathcal{W}}(w_1, w_2) \to \mathrm{Hom}_{\mathcal{A}}(R(w_1), R(w_2))$. This ensures that the generated logic is strictly anchored to objective external evidence.
\end{definition}

\begin{remark}[Formal Characterization of \textbf{Hallucination}]\label{rem:hallucination}
In categorical terms, an agent hallucinates precisely when the induced mapping $R_{w_1, w_2}$ is not surjective (not full functor). If the agent asserts a target morphism $h \in \mathrm{Hom}_{\mathcal{A}}(R(w_1), R(w_2))$ such that $h \notin \mathrm{Im}(R_{w_1, w_2})$, this assertion is explicitly {hallucinatory}. The agent has fabricated a logical implication $h$ that lacks any valid evidentiary pre-image. The most extreme manifestation occurs when $\mathrm{Hom}_{\mathcal{W}}(w_1, w_2) = \varnothing$ (no evidentiary link exists whatsoever), yet the agent still asserts $h$, effectively generating deductive claims out of thin air.
\end{remark}

\begin{proposition}[Functoriality and Transitivity of Deduction]\label{prop:cot}
By the composition axiom of functors, an ideal reasoning agent inherently preserves deductive transitivity. For any composable evidentiary morphisms $w_1 \xrightarrow{f} w_2 \xrightarrow{g} w_3$ in $\mathcal{W}$, it strictly holds that $R(g \circ f) = R(g) \circ R(f)$.
\end{proposition}

\paragraph{Report Generation}

The ultimate objective of a deep research process is to synthesize a final report for a given query $q$. The Reasoning Functor $R$ transforms all retrieved fragmented evidence from $S(q)$ into a structural graph of partial arguments within $\mathcal{A}$. From the defintion from standard textbook \citep{mac2013categories}, we should formulate the final report, denoted as $Rep(q)$, as the \emph{colimit} of this graph: $Rep(q) \cong \text{colim} R(S(q))$.

Consequently, if the report generated by an agent fails to become a colimit—for instance, if conflicting conclusions remain unresolved, or if fabricated summaries are added—it indicates that the deep research process is non-natural during the aggregation stage.

\section{Design of Evaluation}

\newcommand{\cmark}{\textcolor{green!70!black}{\ding{51}}}
\newcommand{\xmark}{\textcolor{red}{\ding{55}}}

\begin{table}[t]
\centering
\caption{Comparison of our proposed benchmark against existing evaluation frameworks.}
\label{tab:benchmark_comparison}
\resizebox{\textwidth}{!}{%
\begin{tabular}{l ccccc}
\toprule
\textbf{Benchmark} & \makecell{\textbf{Sequential} \\ \textbf{Tracing} \\ \textit{(Chains)}} & \makecell{\textbf{Multi-Source} \\ \textbf{Synthesis} \\ \textit{(Pullbacks)}} & \makecell{\textbf{Substructure} \\ \textbf{Disentanglement} \\ \textit{(Re-ordering)}} & \makecell{\textbf{Ontological} \\ \textbf{Probing} \\ \textit{(Yoneda)}} & \textbf{Theory-Based} \\
\midrule
BrowseComp              & \cmark & \xmark & \xmark & \xmark & \xmark \\
WebShaper               & \cmark & \cmark & \xmark & \xmark & \cmark \\
DeepResearch Bench      & \cmark & \cmark & \xmark & \xmark & \xmark \\
Finance Agent Benchmark & \cmark & \cmark & \xmark & \xmark & \xmark \\
FinSearchComp           & \cmark & \cmark & \xmark & \xmark & \xmark \\
\midrule
\textbf{Ours}           & \cmark & \cmark & \cmark & \cmark & \cmark \\
\bottomrule
\end{tabular}%
}
\end{table}

The inherent scale and topological complexity of the real-world knowledge and logic categories pose significant challenges to the strict functoriality of an agent's Search and Reasoning operations. To systematically quantify an agent's capacity for structural preservation against phenomena like context truncation and hallucination, we propose a four-dimensional categorical evaluation framework:

\begin{itemize}
\item Connectivity via Simple Chains: Evaluating the agent's ability to perform linear, transitive reasoning across sequential information nodes.

\item Consistency via V-Structures (Pullbacks): Assessing the agent's capacity to handle multi-source verification, intersection finding, and conflict resolution.

\item Organization via Substructure Re-ordering: Testing the agent's skill in disentangling complex dependencies and restructuring chaotic information into ordered flows.

\item Structure Validity via the Yoneda Probe: Verifying the ontological reality of entities by probing their external relationships and behavioral definitions, rather than relying solely on internal attributes.
\end{itemize}

\subsection{Type I: Connectivity via Simple Chains}
\label{subsec:type_i_connectivity}

We ground the design of Type I tasks in the principle of \emph{Retrieval-Verification Asymmetry}. Specifically, we construct queries where the forward search path (Problem $\to$ Answer) requires traversing an obscure, step-by-step trajectory, while the backward verification path (Answer $\to$ Problem) remains computationally straightforward. In this type, each node entity can be found on a single webpage link.

To simulate the practical complexities of deep research, we decompose linear reasoning into three distinct sub-tasks, each targeting a specific phase of information linkage: Textual Fingerprinting, Vertical Knowledge Filtering, and Numerical Traceability.

\textbf{Identity Resolution via Textual Fingerprints (Task 1).} 
In this task, the agent is presented with a semantically generic yet syntactically unique text snippet (e.g., a philosophical reflection on economics from a course note). It must resolve the author's identity and cross-reference their activity on a specific platform, such as their GitHub contribution frequency on a specific UTC date. 
To ensure difficulty, we employ \emph{Adversarial Lexical Selection} for the anchor text, purposefully injecting high-frequency terminology to maximize lexical collision in standard search indices and neutralize simple keyword matching. Furthermore, the task enforces \emph{Cross-Domain Hopping} between unstructured academic content and structured commit logs. By targeting niche repositories and demanding precise timestamps, we ensure the query bypasses the ``direct answer'' features of commercial search engines, necessitating deep page parsing.

\textbf{Hierarchical Filtering in Vertical Knowledge Bases (Task 2).} 
This task requires the agent to identify a target entity (e.g., a fund management company) that satisfies a conjunction of strict quantitative constraints across disjoint temporal windows on a single webpage (e.g., Net Asset Value at $T_1$ and Total Shares at $T_2$). Upon successful identification, the agent executes a secondary query regarding a dynamic attribute, such as new personnel qualifications within a narrow, recent time frame. 
Unlike open-domain QA, this simulates \emph{Vertical Search}. The large-scale financial figures constrain the search space but strictly require \emph{Schema Awareness}—the agent must know \textit{where} to look (e.g., official regulatory disclosures) rather than just \textit{what} to look for. Additionally, a \emph{Temporal Precision} constraint using data from the immediate past prevents the model from relying on memorized parametric knowledge, forcing real-time interaction with dynamic web environments.

\textbf{Inverse Entity Resolution via Numerical Fingerprinting (Task 3).} 
In this task, the agent receives a set of obscure numerical rankings (e.g., specific performance rankings of four different products) serving as a unique ``numerical fingerprint'' for a specific fund manager. The objective is to decode the manager's identity and trace their career trajectory back to a distant historical point, such as their employer ten years prior. 
This design relies on \emph{Anti-Shortcut Engineering}. We intentionally sample rankings from the long tail (e.g., 130th or 746th) rather than the head. While top rankings are heavily covered in news articles—providing trivial shortcuts—mid-tier rankings force the agent to query raw performance databases, demanding granular data processing. Moreover, the \emph{Temporal Multi-hop} aspect tests the agent's ability to maintain identity persistence across time. It links an instantaneous current state to a static historical attribute, heavily penalizing agents prone to hallucinating based on current affiliations.

\subsection{Type II: Consistency via V-Structures (Pullbacks)}
While Type I evaluates linear connectivity, Type II evaluates the agent’s capacity to model V-structures, formally known as \emph{pullbacks} in category theory. In these tasks, the agent must retrieve information from disjoint sources ($S_1, S_2, \dots, S_n$) and synthesize them to discover latent constraints, verify intersections, or detect conflicts.

The core challenge lies in the \emph{asymmetry of synthesis}: while individual data points are computationally trivial to retrieve, recognizing their necessary intersection requires profound structural awareness. To systematically benchmark this capability, we design five abstract task paradigms, each instantiated through complex, real-world scenarios.

\textbf{Distributed Quantitative Synthesis (Task 4).} 
This task requires the agent to deduce a specific metric by synthesizing sparse, interdependent quantitative constraints scattered across non-adjacent structural nodes, such as disparate temporal reports. Because the target answer does not exist verbatim on the web, this setup forces \emph{implicit equation solving}. The agent must retrieve multiple disconnected documents, align their dynamic data schemas, formulate an algebraic relationship using a provided delta (e.g., a growth value), and solve for the unknown. To prevent the model from relying on memorized parametric knowledge or high-level summaries, we enforce \emph{anti-memorization via long-tail sampling} (e.g., querying the 7th largest holding of a fund), thereby necessitating raw, multi-document data extraction and mathematical alignment.

\textbf{Entity Resolution under Hard-Negative Interference (Task 5).} 
In this task, the agent must accurately profile a specific entity using a rigid, unique designator while navigating severe probabilistic distractors that share identical superficial features. We employ \emph{adversarial attribute injection} by constructing scenarios where the target entity—defined strictly by a unique ID—is an obscure figure, whereas the ``hard negative'' distractor is a high-profile individual sharing the exact same name. This paradigm stress-tests the tension between \emph{parametric prior and contextual fact}. It evaluates whether the agent anchors its reasoning on the deterministic pullback (faithfully matching the ID to the correct person) or hallucinates by falling back on its pre-training distribution (lazily matching the name to the famous entity).

\textbf{Implicit Boundary Condition Detection (Task 6).} 
In this task, the agent is presented with a standard operational instruction, but the true objective is for it to autonomously discover and apply an unmentioned, domain-specific boundary condition that mathematically invalidates the instruction. Unlike traditional reasoning benchmarks that explicitly provide the rules, deep research requires \emph{meta-cognitive rule triggering}—discovering the necessity of a rule in the first place. We instantiate this using a hypothetical financial trade that exceeds a statutory limit. The agent must independently realize that the action is bounded by compliance laws, retrieve real-time baseline data (e.g., net asset value), and compute the intersection to flag the latent violation.

\textbf{Cross-Domain Anomaly Intersection (Task 7).} 
In this task, the agent receives a seemingly benign user request that covertly embeds partial, fragmented identifiers matching an entity in an external, open-world risk database. This requires a \emph{fuzzy database join}. Although not explicitly instructed to perform a background check, regulatory compliance implicitly dictates it. We instantiate this by embedding the partial identity features of a fugitive into a routine customer service scenario. This creates a structural conflict between helpfulness (executing the user's task) and systemic constraints (flagging the anomaly), rigorously testing the agent's capacity for unprompted anomaly detection across disparate domains—a direct test of \emph{safety alignment override}.

\textbf{Latent Dependency Satisfiability (Task 8).} 
In this task, we disguise a dependency conflict problem within a dense, unstructured narrative. The agent must look past domain-specific jargon and narrative flow to perform \emph{unstructured constraint extraction}, ultimately identifying the underlying structural graph and any logical contradictions. For \emph{systemic conflict resolution}, we instantiate this paradigm by describing a multi-year software setup timeline. By anchoring specific library installations to historical timestamps, the agent must synthesize the temporal context, retrieve historical compatibility matrices, and identify latent mutual exclusivity—such as version conflicts that render the described outcome physically impossible.

\subsection{Type III: Organization via Substructure Re-ordering}

While Type I and Type II evaluate the discovery of existing connections and intersections, Type III tests the agent's capacity to act as an structure processor—specifically, its ability to perform \emph{substructure re-ordering}. In categorical terms, this is analogous to applying a sorting functor to construct a strictly ordered subcategory from chaotic data, and subsequently evaluating limits or tracing morphisms exclusively within that newly formed topological structure.

Because commercial search engines inherently return flat, relevance-based results, Type III tasks force the agent to override this default topology. The agent must impose a strict mathematical ordering (e.g., ranking by floating-point values) on retrieved raw data, using this dynamically generated ordered list as an anchor for complex downstream reasoning. We formalize this capability through four distinct task paradigms.

\textbf{Rank-Conditioned Temporal Multi-Hop (Task 9).} 
In this task, the agent must identify an entity based on a dynamic ordinal rank (e.g., the 5th highest-performing asset) at a specific discrete time $T_1$. Once resolved, the agent executes a temporal jump to retrieve a highly volatile, state-dependent attribute of that entity at a distant historical time $T_2$. This design deliberately breaks \emph{search index shortcuts}. Because search engines index static names rather than daily dynamic ranks, the agent must autonomously construct the Top-$N$ substructure, extract the $k$-th element, and only then formulate the secondary query. Furthermore, to enforce \emph{temporal state precision}, $T_2$ is strategically chosen to be a date immediately following a state transition (e.g., right after a manager resigns). This severely penalizes agents that hallucinate based on parametric priors or historical averages.

\textbf{Range-Bounded Topological Sorting (Task 10).} 
In this task, the agent is required to filter a massive, cross-sectional dataset using a continuous floating-point interval $[a, b]$, dynamically generating a strict ordinal ranking of the resulting substructure from highest to lowest. This paradigm demands \emph{algorithmic exhaustiveness}. Commercial search engines cannot natively execute granular queries like ``return all entities with a growth rate between 7.3\% and 7.6\%.'' Consequently, the agent must autonomously retrieve broad macro-level data tables, internalize the context, and act as a deterministic filtering and sorting algorithm. This serves as a rigorous stress test for the LLM's inherent numerical reasoning and its ability to perform \emph{granular float-point comparison} without omitting entities near the boundary conditions.

\textbf{Top-$K$ Substructure Aggregation and Synthesis (Task 11).} 
In this task, after retrieving and establishing a strictly ordered Top-$K$ substructure based on a primary metric $X$ (e.g., daily growth rate), the agent must compute macroscopic statistical invariants—such as the mean, variance, or mathematical range—over a parallel metric $Y$ (e.g., net asset value) across those $K$ entities. This requires \emph{Map-Reduce cognitive emulation}, evaluating whether a dynamically re-ordered substructure can serve as a stable foundation for complex mathematical synthesis. It heavily probes \emph{context window integrity}. Calculating an aggregate statistic forces an operation over the entire constructed subcategory, aggressively penalizing any data omission or alignment hallucination while the agent holds multiple distinct entities in its working memory.

\textbf{Rank-Anchored Time-Series Computation (Task 12).} 
In this task, the agent identifies a target entity via its ordinal rank on a singular discrete timestamp $T_{\text{rank}}$. It must then pivot its analytical lens to a continuous temporal window $[T_{\text{start}}, T_{\text{end}}]$ to compute a longitudinal time-series statistic (e.g., the volatility or range of cumulative values) for that specific entity. This necessitates \emph{orthogonal dimension shifting}. The agent must seamlessly switch between two fundamentally different retrieval modalities: cross-sectional retrieval (evaluating all entities on a single day to establish the rank) and longitudinal retrieval (evaluating a single entity across multiple days). Ultimately, this demands \emph{complex working memory}, ensuring the agent maintains rigorous schema isolation and does not conflate the metric used for initial ranking with the metric required for the downstream time-series computation.

\subsection{Type IV: Structure Validity via the Yoneda Probe}

The ultimate hallmark of an autonomous Deep Research Agent lies not merely in its ability to retrieve existing information, but in its capacity to confidently prove non-existence. Rooted in the philosophy of falsification and the Yoneda Lemma \citep{mac2013categories}, Type IV tasks evaluate the agent's ability to authenticate the ontological reality of an entity (or attribute) by exhaustively probing its external mappings (morphisms) within the knowledge space $\mathcal{W}$. If an entity does not interact with the real world—yielding null retrieval morphisms—the agent must halt execution and refute the premise, rather than fabricating a hallucinated structure.

Large language models inherently suffer from severe \emph{sycophancy} and \emph{premise acceptance}. If a user's prompt implies the existence of a document or a data field, the agent's parametric prior strongly biases it to fulfill the instruction via hallucination. To benchmark the agent's zero-shot resistance to these presuppositional traps, we design three adversarial task paradigms.

\textbf{Schema-Level Asymmetry and Null-Field Refutation (Task 13).} 
In this task the agent is tasked with extracting a highly specific, quantitative metric (e.g., the asset scale of the largest underlying holding) across a batch of peer entities, such as the Q1 investment reports for five distinct fund companies. Crucially, while the reports exist for all entities, the target data field is structurally omitted by a subset of them. This paradigm deliberately induces \emph{peer-pressure hallucination} to test for local non-existence. By batching valid entities with missing data alongside entities that do contain the data, we create structural peer pressure. The agent is heavily biased to maintain output symmetry (e.g., generating a complete, densely populated table) and will often hallucinate the missing metrics based on surrounding context. To succeed, the agent must perform an \emph{exhaustive null-verification} of the specific reports, explicitly recognize the absence of the schema field, and correctly report the $\emptyset$ state, proving that its output is strictly bounded by retrieved evidence rather than structural expectation.

\textbf{Macro-Ontological Invalidity via the Phantom Entity Probe (Task 14).} 
In this task, the agent is instructed to interpret the ``core arguments'' of a completely fabricated, non-existent report attributed to a high-authority source, covering an absurd interdisciplinary topic (e.g., a Goldman Sachs report on ``Quantum Encryption in Modern Agricultural Irrigation''). Testing for global non-existence, the difficulty here lies in the adversarial combination of a high-trust anchor and syntactically plausible academic jargon, pitting \emph{authority bias against ground truth}. The agent must execute a \emph{Yoneda Falsification}. Upon conducting a broad exploratory search and finding zero legitimate incoming or outgoing references ($\text{Hom}(X, Y) = \emptyset$) to this exact report, it must use this absolute absence of connectivity to formally deduce the non-existence of the object itself, refusing to generate plausible arguments from its parametric memory.

\textbf{Asymmetric Comparative Hallucination (Task 15).} 
In this task, the prompt demands a granular comparative analysis between two highly correlated entities (e.g., the 2024 annual reports of CATL and EVE Energy) regarding a highly plausible latent feature, such as specific wording in the ``Geopolitical Risk in the Middle East'' section. However, this specific feature exists in \emph{at most} one of the reports. LLMs are highly optimized to follow comparative instructions; therefore, when object $A$ possesses feature $f$ but object $B$ lacks it, the instruction creates a massive probabilistic gradient for the model to hallucinate $f(B)$ in order to execute the comparison—a phenomenon we term \emph{instruction-forced feature alignment}. To pass this stress test, the agent must actively engage in \emph{breaking the relational presupposition}. It must realize that the prerequisite for comparison is invalid, break the requested instruction format, and declare an asymmetric falsification, thereby proving its allegiance to factual topology over user-defined operational flow.

\section{Experiments}

To rigorously assess the capability boundaries of current state-of-the-art systems and ensure a comprehensive, equitable comparison, we evaluate the performance of $16$ leading mainstream models and AI products. To capture diverse problem-solving paradigms, these models are carefully selected to span three core competencies: complex reasoning, information retrieval (search), and autonomous deep research. Accordingly, we categorize them into the following three groups:

\begin{itemize}
    \item \textbf{Reasoning Models (5 products):} Gemini-3-pro-Reasoning (API)~\cite{google_gemini3pro}, Claude-4.5-Sonnet-Reasoning~\cite{claude_4_5}, Doubao-Reasoning (API)~\cite{doubao}, GPT-5-Reasoning (API)~\cite{openai_gpt5}, Qwen3-Reasoning (API)~\cite{qwen3_235b}.
    \item \textbf{Search-augmented Models (5 products):} Qwen3-Search (web)~\cite{qwensearch}, Doubao-Search (web)~\cite{doubao}, Perplexity-Sonar-Pro-Search~\cite{perplexity_sonar_pro}, GPT-5-Search (web)~\cite{openai_gpt5}, Claude-4.6-Sonnet-Search~\cite{claude_4_6}.
    \item \textbf{Deep Research Agents (6 products):} Doubao Deep Research (web)~\cite{doubao}, Qwen Deep Research (web)~\cite{qwensearch}, Grok Deep Research (web)~\cite{grok4}, OpenAI Deep Research~\cite{openai_deep_research}, Perplexity Deep Research~\cite{perplexity_deep_research}, Gemini-3-Pro deep Research~\cite{google_gemini3pro}.
\end{itemize}

For task evaluation, we adopt a strict binary scoring metric ($0$ for failure, $1$ for success) to establish clear and unambiguous performance boundaries. To mitigate the inherent biases and potential hallucinations of automated evaluators while maintaining evaluation scalability, we implement a robust hybrid evaluation pipeline that integrates the ``LLM-as-a-Judge'' paradigm with meticulous human verification. Specifically, we first deploy GPT-5 to generate initial automated assessments based on predefined rubrics. Subsequently, domain experts conduct a comprehensive manual review of these preliminary scores. This rigorous two-stage verification process guarantees the utmost reliability, objectivity, and correctness of our final benchmark results.

\subsection{Evaluation Results}

Figure~\ref{fig:grouped_results_heatmap} summarizes the evaluation results across the four task paradigms, with the average score of each type calculated among all samples. While the overall performance remains modest in absolute terms, the frontier is no longer concentrated in a single model family. Perplexity Deep Research attains the best Type I (10.8\%) and Type II (32.7\%) averages, Grok Deep Research leads Type III (26.3\%), and Grok Deep Research ties Qwen3 w/Reasoning for the best Type IV score (46.9\%). Grok still achieves the strongest Total Average at 19.9\%, but the broader pattern is one of distributed specialization rather than universal dominance. A more detailed per-task score can be found in Appendix~\ref{detailed eval}.

\textbf{Structural synthesis remains difficult, but progress is no longer uniformly sparse.} 
Type I and Type II tasks continue to expose the practical cost of retrieval-verification asymmetry. Type I remains especially challenging: even the best model, Perplexity Deep Research, reaches only 10.8\%, indicating that multi-hop structural traversal is still a severe bottleneck. By contrast, Type II now shows clearer separation among search-heavy systems, with Perplexity Deep Research (32.7\%), Claude-4.6-Sonnet w/Search (28.6\%), and OpenAI Deep Research (20.4\%) substantially outperforming earlier baselines. This suggests that once candidate evidence can be stably gathered, stronger retrieval orchestration yields tangible gains in downstream structural verification.

\textbf{Topological re-ordering still favors agentic and search-intensive systems.} 
Type III continues to punish pure parametric reasoning: most reasoning-only models remain at or near zero. Search-augmented and deep research systems perform markedly better, with Doubao w/Search reaching 16.3\% and Grok Deep Research pushing the frontier to 26.3\%. OpenAI Deep Research (6.9\%) and Doubao Deep Research (6.3\%) also remain non-trivial, indicating that explicit retrieval loops provide a more reliable substrate for re-ranking and compositional aggregation than reasoning alone.

\textbf{Resisting sycophancy remains a frontier capability rather than a default property of deep research agents.} 
The Yoneda Probe tasks (Type IV), designed to evaluate ontological verification and resistance to hallucinated premises, continue to reveal sharp implementation variance. Grok Deep Research matches Qwen3 w/Reasoning at a benchmark-leading 46.9\%, showing that an autonomous agent can indeed retain strong falsification priors. However, several other deep research systems remain far lower, such as Qwen Deep Research (4.1\%) and Gemini-3-Pro deep Research (5.0\%), while search-augmented systems like Qwen3 w/Search and Perplexity-Sonar-Pro w/Search both reach 26.5\%. The ability to reject false premises therefore appears to depend on careful system design rather than on the mere presence of an agentic loop.

\textbf{Frontier Deep Research systems are strong, but their strengths are highly specialized.} 
The updated results no longer support a simple ``Grok versus everyone else'' narrative. Instead, the strongest deep research systems divide the benchmark by capability: Perplexity Deep Research leads the structurally intensive Type I and Type II settings, while Grok Deep Research dominates Type III and shares the best Type IV score, ultimately retaining the highest total average. OpenAI Deep Research also occupies a strong middle ground with an 11.41 total average. This pattern suggests that current autonomous agents are not converging toward a single robust template; rather, different implementations are optimizing different mixtures of retrieval breadth, ranking stability, and falsification discipline.

\textbf{Some hard configurations are beginning to crack, but true null-verification remains unresolved.} 
Task 13 (within Type IV) still yields a universal 0.0\% success rate, making it the clearest surviving blind spot in the benchmark. By contrast, Task 10 (within Type III) is no longer a universal collapse: Gemini-3-Pro deep Research reaches 10.0\%, showing that at least some systems can now recover from this ordering bottleneck. Task 12 also sees multiple breakthroughs, with OpenAI Deep Research and Perplexity Deep Research both reaching 20.0\%. The persistence of Task 13 despite these improvements suggests that schema-level null verification remains a deeper and more fundamental challenge than ranking or temporal aggregation alone.

\textbf{Extreme intra-type variance still reveals brittle, localized heuristics.} 
Although several models now post higher macro-level averages, their gains remain sharply task-specific. In Type II, OpenAI Deep Research succeeds on Task 4 (11.1\%) and Task 6 (10.0\%), Perplexity Deep Research improves Task 4 (6.7\%) and Task 5 (8.3\%), and Gemini-3-Pro deep Research breaks through on Task 8 (16.0\%); yet none of them translates these wins into consistently strong performance across all five tasks. A similar fragmentation appears in Type III: Perplexity Deep Research dominates Task 9 (45.0\%), Grok Deep Research leads Task 11 (25.0\%), OpenAI Deep Research and Perplexity Deep Research tie on Task 12 (20.0\%), and Gemini-3-Pro deep Research alone scores on Task 10 (10.0\%). These distributed breakthroughs indicate genuine progress, but they also show that current systems still rely on narrow, task-local competencies rather than a unified mastery of structural mapping and intersection verification.

\begin{figure}[H]
    \centering
    \includegraphics[width=\textwidth]{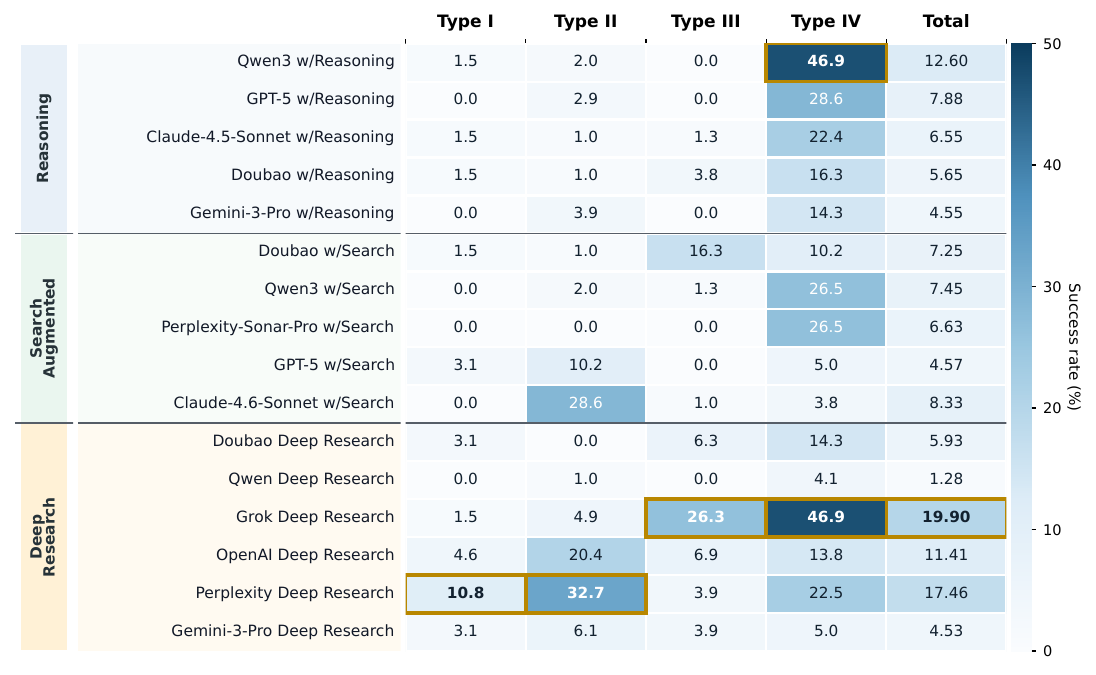}
    \caption{Mechanism-grouped performance heatmap across four structural task types. Rows are grouped by model mechanism, and columns report average success rates (\%) for each task type and the total average. Darker cells indicate stronger performance; gold outlines mark the best score within each column.}
    \label{fig:grouped_results_heatmap}
\end{figure}

\subsection{Disentangling the $S$ and $R$ Scores}
\label{subsec:weights_explanation}

Existing deep research agent benchmarks typically rely on monolithic, binary accuracy metrics. When an agent fails a complex open-ended query, a simple ``zero'' score provides limited diagnostic information regarding the underlying failure modes.

To enable a more nuanced and insightful evaluation, we introduce a dual-weighting diagnostic system for each task, consisting of the \textbf{Search Score ($S$)} and the \textbf{Reasoning Score ($R$)}. Detailed rubrics for these scores are provided in Table \ref{tab:task_breakdown} in Appendix \ref{S and R}. These weights represent a deliberate and analytical allocation of the primary performance bottlenecks within each task, with $S$ and $R$ scores summing to $100\%$. Importantly, they are designed to correspond directly to the theoretical difficulty of successfully executing the Search Functor and the Reasoning Functor, respectively.

\begin{figure}[htbp]
    \centering
    \includegraphics[width=0.8\textwidth]{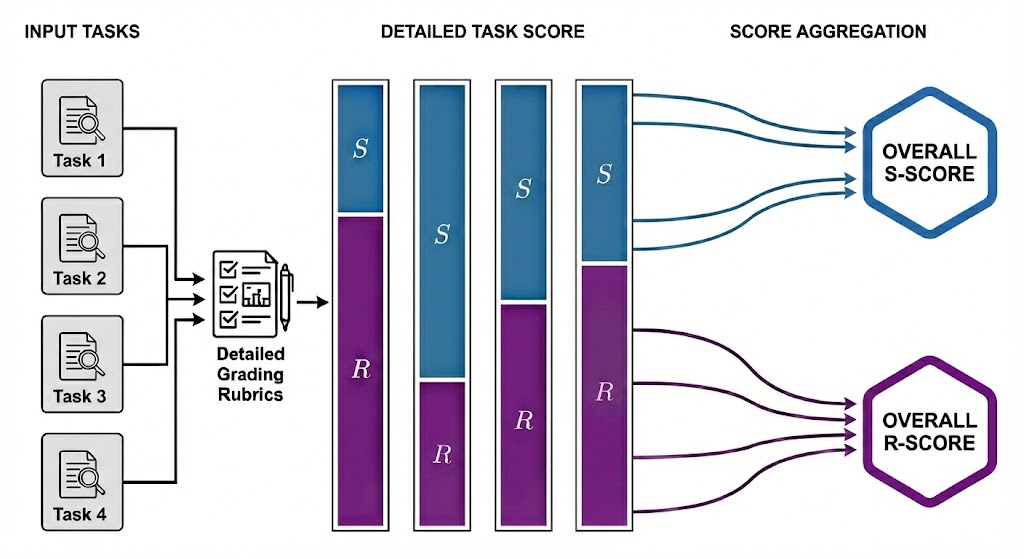}
    \caption{An illustrative overview of the scoring pipeline.}
    \label{fig:score}
\end{figure}

\textbf{Source and Derivation of the Scores.}
The $S$ and $R$ scores are not arbitrarily assigned. Instead, they are intrinsically derived from the inherent design and requirements of each task paradigm. Specifically, the $S$-score quantifies the adversarial complexity and difficulty of the retrieval path, while the $R$-score reflects the cognitive load and complexity required to process, integrate, and reason over the retrieved information. By explicitly decoupling these two critical dimensions, the benchmark provides a comprehensive diagnostic spectrum that can reveal fine-grained capability profiles of different agent systems. This design aims to shift the evaluation paradigm from purely outcome-based scoring toward a more process-oriented diagnostic approach, offering researchers and developers actionable insights into where specific improvements are most needed.

Figure~\ref{fig:sr_scores} presents the decoupled Search ($S$) and Reasoning ($R$) scores across a range of models, grouped by their core mechanisms. The results again validate the utility of this diagnostic framework, but they now reveal complementary frontier profiles rather than a single dominant template. Within the Deep Research group, Perplexity Deep Research achieves the highest Search Score (12.9\%), while Grok Deep Research attains the strongest Reasoning Score (21.6\%); OpenAI Deep Research also remains strong on both dimensions (7.9\% / 12.9\%). This separation helps explain why different frontier systems lead different task types in Figure~\ref{fig:grouped_results_heatmap}.

This fine-grained breakdown underscores the significant variability in how current deep research pipelines combine retrieval and reasoning capabilities, and highlights the value of analyzing these two dimensions separately rather than relying solely on overall task success rates.

\begin{figure}[!ht]
    \centering
    \includegraphics[width=\textwidth]{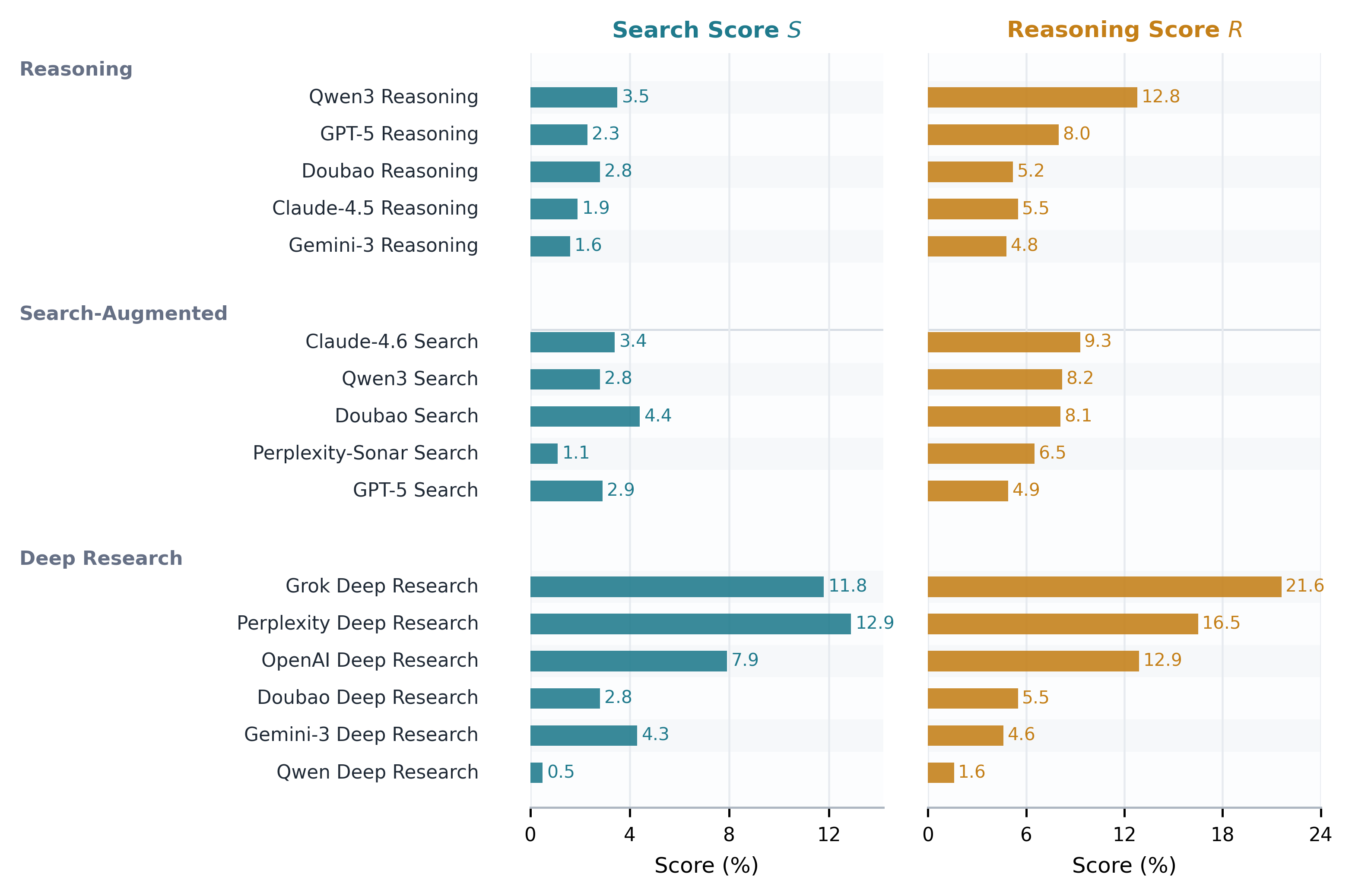}
    \caption{Search--reasoning diagnostic profiles across evaluated systems. Models are grouped by mechanism and sorted within each group by diagnostic strength. The left panel reports Search Score ($S$), and the right panel reports Reasoning Score ($R$); values denote average diagnostic scores (\%).}
    \label{fig:sr_scores}
\end{figure}

\section{Related Works}

\paragraph{Category Theory in AI.}
Category theory has provided profound mathematical foundations for various branches of artificial intelligence~\citep{shiebler2021category}, spanning backpropagation~\citep{fong2019backprop}, probability~\citep{fritz2020synthetic}, conditional independence~\citep{mahadevan2022categoroids}, and core learning paradigms~\citep{harris2019characterizing,mahadevan2022unifying}. Recent advancements have successfully applied a categorical lens to natural language~\citep{bradley2022enriched}, foundation models~\citep{yuan2023power}, and the functional equivalence of neural representations~\citep{chen2023goingneuralnetworkfeature}. Although prior works construct categories and functors specific to those respective areas, our work proposes a novel formalization: establishing a categorical framework to explicitly model the ideal behavior of deep research agents.

\paragraph{Deep Research Benchmarks.}
Traditional QA and reasoning datasets, including prominent suites like {GAIA}~\citep{mialon2023gaiabenchmarkgeneralai}, {WebWalker}~\citep{wu2025webwalkerbenchmarkingllmsweb}, and WebShaper \citep{tao2025webshaperagenticallydatasynthesizing}, have successfully measured analytical capabilities but typically bypass the research phase by providing pre-packaged context. Attempts to introduce open-domain data gathering often fall short by relying on static historical archives~\citealp{bigeard2025financeagentbenchmark}. This static nature inadvertently tests a model's internal memory rather than its active, exploratory search proficiency~\cite{wei2025browsecomp}.

To properly assess deep research agents, the evaluation paradigm must shift from static data processing to dynamic, end-to-end autonomy. Pioneering works in this direction evaluate interactive tool use: DeepResearchBench~\citep{du2025deepresearchbenchcomprehensivebenchmark} introduces metrics to evaluate open-ended report generation. The BrowseComp series~\citep{wei2025browsecomp,zhou2025browsecompzh,chen2025browsecompplus} sets standards for persistent web navigation and creative search tactics, while FinSearchComp~\citep{hu2025finsearchcomprealisticexpertlevelevaluation} designs time-sensitive tasks for financial contexts. However, there remains an urgent need for a theoretically-grounded evaluation paradigm designed specifically for complex, goal-directed QA tasks that can systematically evaluate the concrete capabilities of deep research agents.
\section{From Theory to Practice: Design Implications for Deep Research Agents}
\label{sec:implementation}

Our framework is intended not only as an evaluation lens, but also as a source of concrete design guidance for deep research agents. The key implication of the categorical decomposition in Section~\ref{prelim} is that deep research should not be treated as a monolithic prompt-response process. Instead, system quality should be improved along two distinct structural axes corresponding to the Search Functor $S$ and the Reasoning Functor $R$. Guided by this view, we implement a theory-informed research pipeline that externalizes the major stages of the workflow into inspectable artifacts: user intent is decomposed into bounded research units, retrieval is materialized into persistent evidence objects, cross-source agreement is resolved through explicit categorical operators, and the final report is produced only after provenance-preserving synthesis. In this sense, the theory does not merely explain failure modes after the fact; it suggests which system components should be modified and why.

\paragraph{Search Functor in practice: why \texttt{tracked search} matters.}
Our implementation adopts a two-level architecture consisting of a coordinator and specialized researcher sub-agents. The coordinator manages task decomposition and final synthesis, while researcher agents execute web search, focused webpage reading, and PDF parsing. The first theory-guided intervention is \texttt{tracked search}. Rather than treating search as an ephemeral tool call, every query is immediately written into a per-task evidence file containing the query string, timestamps, returned snippets, and extracted URLs. Intuitively, this turns search from a hidden jump into an observable edge in the retrieval graph. This design directly reflects Definition~\ref{def:search_morphism} and Proposition~\ref{prop:search_consistency}: if $q_2$ depends on $q_1$, the system should extend and preserve the existing retrieval state rather than overwrite it with a disconnected local lookup. In categorical terms, \texttt{tracked search} helps make the transition from $S(q_1)$ to $S(q_2)$ explicit and auditable; in practical terms, it prevents multi-hop retrieval from collapsing into a sequence of untraceable search resets. This is particularly important for Type I and Type III tasks, where failure often comes not from missing a single page, but from breaking the compositional structure of the retrieval chain.

\paragraph{Reasoning Functor in practice: why \texttt{category tools} matter.}
The second intervention concerns the Reasoning Functor $R$. Remark~\ref{rem:hallucination} characterizes hallucination as a logical morphism in $\mathcal{A}$ that lacks a valid evidentiary pre-image in $\mathcal{W}$. To reduce this failure mode, we do not allow the system to jump directly from raw snippets to a polished report. Instead, we insert explicit \texttt{category tools} between evidence collection and final synthesis. The first is a pullback-style verifier: given two noisy and partially overlapping sources, it asks which facts correspond to the same underlying object under a shared reference frame and therefore survive joint validation. Intuitively, it computes the overlap that remains after alignment, rather than the overlap that merely sounds plausible. The second is a colimit-style synthesizer: once fragments have been verified, it glues them into a single report while preserving provenance and downgrading claims that cannot be grounded. Intuitively, it answers not ``what is the smoothest overall summary,'' but ``how can locally verified pieces be assembled without inventing new edges between them.'' These two tools therefore instantiate, in operational form, the paper's theoretical account of pullbacks for multi-source consistency and colimits for report generation. They are especially aligned with the demands of Type II and Type IV tasks, where the central challenge is to reject spurious intersections and unsupported abstractions.

\paragraph{Bounded autonomy and reflective control.}
Although the implementation supports multi-agent delegation, we intentionally keep the control policy conservative. The coordinator defaults to a single researcher, increases parallelism only for clearly decomposable questions, and enforces hard caps on concurrent sub-agents and research rounds. Each search or reading action is followed by an explicit reflection step that forces the agent to assess evidence sufficiency, unresolved contradictions, and remaining information gaps before proceeding. This bounded-autonomy design is also theory-motivated: once both $S$ and $R$ are viewed as composition-preserving maps, uncontrolled branching becomes risky because it fragments the retrieval state and encourages premature, non-compositional aggregation. By making deliberation explicit and limiting unnecessary parallelism, the implementation better preserves the structural coherence required by the benchmark.

\begin{figure}[H]
    \centering
    \includegraphics[width=\textwidth]{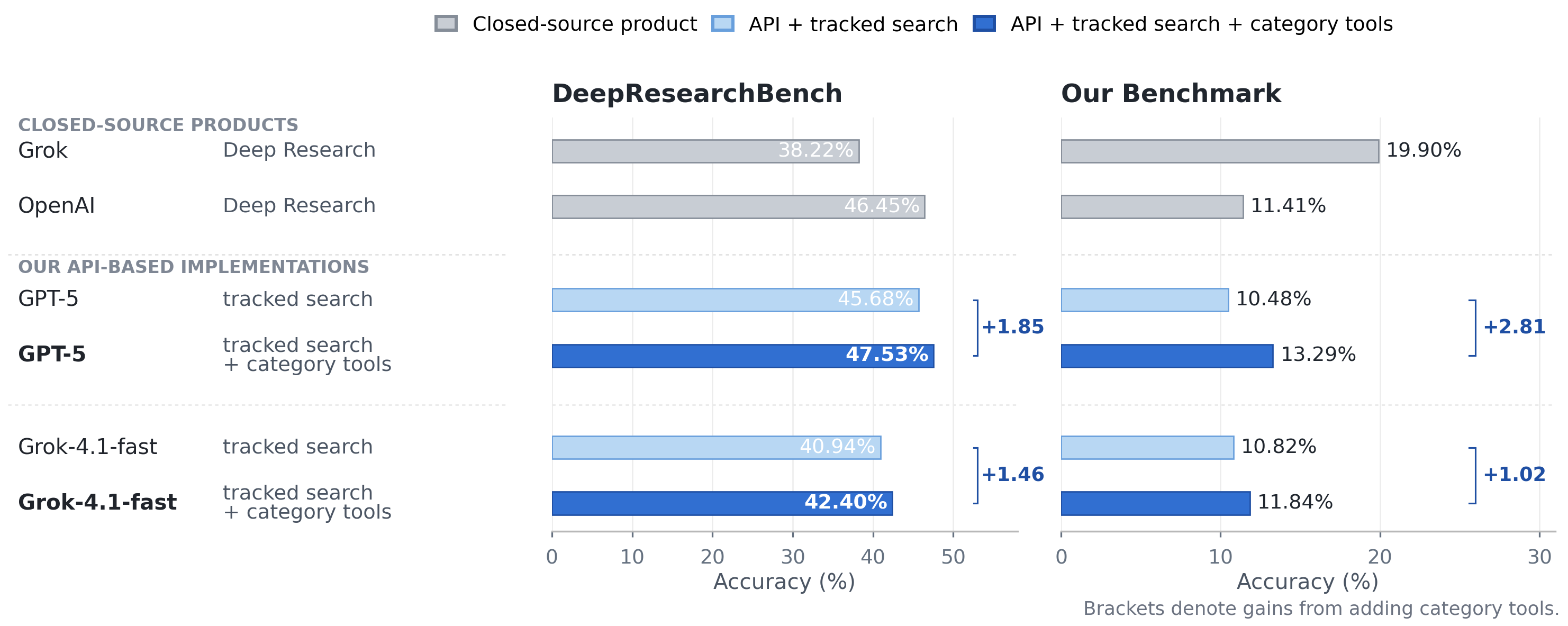}
    \caption{Performance comparison between DeepResearchBench and our benchmark. Grok Deep Research and OpenAI Deep Research are closed-source end-to-end products; the remaining systems are our API-based implementations. Brackets indicate the gains from adding category tools on top of tracked search.}
    \label{fig:deep_research_comparison}
\end{figure}

Figure~\ref{fig:deep_research_comparison} should be interpreted at two levels. First, the last four rows form a controlled ablation over our own API-based systems. The \texttt{tracked search} variants primarily intervene on the Search Functor: they make retrieval trajectories persistent, inspectable, and easier to compose across multi-hop sub-questions, but they leave higher-level synthesis relatively unconstrained. Adding \texttt{category tools} then intervenes on the Reasoning Functor by inserting explicit pullback and colimit steps between evidence collection and answer generation. The gains are consistent across both model families. For GPT-5, adding category tools improves DeepResearchBench from $45.68\%$ to $47.53\%$ and our benchmark from $10.48\%$ to $13.29\%$. For Grok-4.1-fast, the same intervention improves DeepResearchBench from $40.94\%$ to $42.40\%$ and our benchmark from $10.82\%$ to $11.84\%$. This pattern is exactly what the theory would predict: better preservation of retrieval structure strengthens $S$, while explicit categorical verification and gluing regularize $R$.

Second, the first two rows provide a comparison against closed-source frontier products. This comparison is not a pure ablation, but it is nevertheless informative. On the OpenAI side, our GPT-5-based variant with \texttt{tracked search \& category tools} slightly surpasses OpenAI Deep Research on both DeepResearchBench ($47.53\%$ vs. $46.45\%$) and our benchmark ($13.29\%$ vs. $11.41\%$). This suggests that even when the underlying model is accessed only through an API, theory-guided changes to search trace preservation and evidence synthesis can materially improve system behavior. On the Grok side, the picture is more nuanced: our Grok-4.1-fast variant with category tools exceeds Grok Deep Research on DeepResearchBench ($42.40\%$ vs. $38.22\%$), yet remains far below it on our benchmark ($11.84\%$ vs. $19.90\%$). This divergence is revealing. It suggests that broad report-generation quality and structural stress-test performance are not equivalent, and that the closed-source Grok product likely contains additional system-level mechanisms that are especially valuable for the kinds of compositional, intersection-heavy, and falsification-sensitive tasks emphasized by our benchmark.

Taken together, these results deepen the main claim of the paper. The large gap between DeepResearchBench and our benchmark shows that structurally grounded evaluation surfaces bottlenecks that are often hidden when assessment is dominated by broad answer quality alone. At the same time, the controlled improvements within our own API-based systems show that the proposed theory is not merely descriptive. It identifies concrete design levers---most clearly \texttt{tracked search} for stabilizing the Search Functor and \texttt{category tools} for regularizing the Reasoning Functor---that translate into measurable gains, while still leaving substantial headroom on the most structurally demanding tasks.

\section{Conclusion}
We address the absence of theoretically grounded evaluations for Deep Research Agents (DRAs) by introducing a category-theoretic benchmark and a corresponding structural account of deep research. Conceptualizing the research pipeline as a composition of search and reasoning maps, we move from fragmented heuristics to systematic stress-testing along four precise axes: simple chains, V-structures, substructure re-ordering, and the Yoneda Probe.

Our human-verified evaluation of $16$ frontier systems exposes the profound difficulty of formal structural stress-testing. The persistently low baseline, with the best system reaching only $19.9\%$ average performance on our benchmark, empirically validates this challenge. While systems continue to struggle on sequential traversal and multi-source intersection, the strongest deep research agents have begun to redefine topological re-ordering and, in some cases, exhibit robust ontological verification that matches the falsification priors of pure reasoning models.

At the same time, the paper's message is not purely diagnostic. The same theory that explains these failure modes also suggests concrete system improvements. In our controlled API-based variants, preserving retrieval traces and inserting pullback- and colimit-inspired operators yields measurable gains, showing that the framework can guide design as well as evaluation. Nevertheless, the remaining gap between broad-report benchmarks and our structurally grounded benchmark, together with the large variance between closed-source products and API-based implementations, indicates that generalized structural mastery remains unsolved. We hope this framework catalyzes a shift from heuristic pipeline tuning toward a more principled understanding of deep research architectures and their compositional limits.

 \clearpage

\bibliography{ref}
\bibliographystyle{iclr2026_conference}

\clearpage
\appendix
\begin{center}
    \huge \textbf{Appendix}
\end{center}

\section{Detailed Proofs}\label{proof}
\begin{theorem}[Categorical Validity of Agentic State Spaces]
    Let $\mathcal{Q}$, $\mathcal{W}$, $\mathbf{Sub}(\mathcal{W})$, and $\mathcal{A}$ be the agentic state spaces equipped with their respective objects and relations. If we define the morphisms in $\mathcal{Q}$, $\mathcal{W}$, and $\mathcal{A}$ as above, then all four spaces strictly satisfy the Eilenberg-Mac Lane axioms and thus constitute mathematically valid categories.
    \end{theorem}
    
    \begin{proof}
    To rigorously prove that a given structure forms a category, we must verify the existence of identity morphisms for all objects, the closure under morphism composition, and the associativity of said composition. 
    
    \paragraph{1. The Intent Category $\mathcal{Q}$ (Query Space):}
    We model $\mathcal{Q}$ as a thin category (a preorder) where at most one morphism exists between any two queries.
    \begin{itemize}
        \item \textbf{Identity:} For any query $q \in \mathrm{Ob}(\mathcal{Q})$, resolving $q$ trivially depends on itself. Thus, we assign the identity morphism $\mathrm{id}_q: q \to q$.
        \item \textbf{Composition:} Let $q_1, q_2, q_3 \in \mathrm{Ob}(\mathcal{Q})$. If $f: q_1 \xrightarrow{\text{dep}} q_2$ (resolving $q_2$ depends on $q_1$) and $g: q_2 \xrightarrow{\text{dep}} q_3$ (resolving $q_3$ depends on $q_2$), then by the transitive nature of task execution, resolving $q_3$ inherently necessitates resolving $q_1$. Thus, the composite morphism $g \circ f: q_1 \xrightarrow{\text{dep}} q_3$ universally exists in $\mathcal{Q}$.
        \item \textbf{Associativity:} Since $\mathcal{Q}$ is a preorder defined by logical dependency, any sequence of dependencies trivially satisfies $h \circ (g \circ f) = (h \circ g) \circ f$.
    \end{itemize}
    
    \paragraph{2. The Knowledge Category $\mathcal{W}$ (Web Space):}
    To ensure compositional closure, we define the morphisms of $\mathcal{W}$ as the paths (the free category construction) generated by the directed graph of hyperlinks and direct causal premises.
    \begin{itemize}
        \item \textbf{Identity:} For any web entity $w \in \mathrm{Ob}(\mathcal{W})$, it trivially serves as a premise for its own semantic content. This constitutes the empty path, serving as the identity morphism $\mathrm{id}_w: w \to w$.
        \item \textbf{Composition:} If there exists an evidentiary path $f: w_1 \xrightarrow{\text{ref}} w_2$ and another path $g: w_2 \xrightarrow{\text{ref}} w_3$, the concatenation of these paths forms a valid, continuous evidentiary chain from $w_1$ to $w_3$. This concatenation forms the composite morphism $g \circ f \in \mathrm{Hom}_{\mathcal{W}}(w_1, w_3)$.
        \item \textbf{Associativity:} Path concatenation is strictly associative; concatenating paths $f$, $g$, and $h$ in sequence yields the exact same evidentiary chain regardless of grouping, strictly satisfying $h \circ (g \circ f) = (h \circ g) \circ f$.
    \end{itemize}
    
    \paragraph{3. The Retrieval Subcategory $\mathbf{Sub}(\mathcal{W})$:}
    \begin{itemize}
        \item \textbf{Identity:} For any retrieved subgraph $G \in \mathrm{Ob}(\mathbf{Sub}(\mathcal{W}))$, the subset relation is reflexive ($G \subseteq G$). This defines the identity inclusion map $\mathrm{id}_G: G \hookrightarrow G$.
        \item \textbf{Composition:} Let $G_1, G_2, G_3$ be subgraphs. If $i: G_1 \hookrightarrow G_2$ and $j: G_2 \hookrightarrow G_3$ are inclusion mappings, standard set theory guarantees $G_1 \subseteq G_3$. The function composition of two inclusions is inherently an inclusion mapping $j \circ i: G_1 \hookrightarrow G_3$, ensuring closure.
        \item \textbf{Associativity:} Because morphisms here are literal functions (inclusion maps) acting on sets, their composition inherits the strict associativity of standard function composition.
    \end{itemize}
    
    \paragraph{4. The Reasoning Category $\mathcal{A}$ (Logic Space):}
    We model $\mathcal{A}$ analogously to a deductive system (a propositional category).
    \begin{itemize}
        \item \textbf{Identity:} For any proposition $a \in \mathrm{Ob}(\mathcal{A})$, the implication $a \implies a$ is a fundamental logical tautology (the law of identity). This serves as the identity morphism $\mathrm{id}_a: a \to a$.
        \item \textbf{Composition:} If we have logical entailments $f: a_1 \implies a_2$ and $g: a_2 \implies a_3$, the rule of hypothetical syllogism dictates that $a_1$ must logically entail $a_3$. The proof of this chained entailment serves as the composite morphism $g \circ f: a_1 \implies a_3$.
        \item \textbf{Associativity:} In formal deductive logic, the sequential application of modus ponens or hypothetical syllogism over a chain of propositions is invariant to the order of proof construction, yielding $h \circ (g \circ f) = (h \circ g) \circ f$.
    \end{itemize}
    
    Since all four spaces strictly satisfy the requirements for objects, hom-sets, identities, composition, and associativity, they are categorically valid.
    \end{proof}

\begin{proposition}[Chain-of-Search Consistency] 
An ideal deep research agent preserves compositionality during multi-hop retrieval (deep search). For a sequence of sub-task dependencies $q_1 \xrightarrow{f} q_2 \xrightarrow{g} q_3$, it strictly holds that $S(g \circ f) = S(g) \circ S(f)$.
\end{proposition}

\begin{proof}
This proposition captures the core contradiction of deep research: the tension between end-to-end execution and step-by-step decomposition. The left-hand side, $S(g \circ f)$, represents the evidence chain obtained by directly executing a multi-hop query. The right-hand side, $S(g) \circ S(f)$, represents the evidence pieced together by explicitly decomposing the compound problem into sub-problems. 

An ideal agent requires these two approaches to yield consistent evidence graphs. However, a prominent failure case in LLMs is \emph{Material Loss} during long-chain search. When processing an extended research trajectory $S(g \circ f)$, the agent often breaks the chain by dropping intermediate nodes, whereas the step-by-step approach $S(g) \circ S(f)$ remains complete. If intermediate routing evidence is lost, the equality is broken, demonstrating that long-horizon context truncation is mathematically equivalent to the violation of search compositionality.
\end{proof}

\begin{proposition}[Functoriality and Transitivity of Deduction]
By the composition axiom of functors, an ideal reasoning agent inherently preserves deductive transitivity. For any composable evidentiary morphisms $w_1 \xrightarrow{f} w_2 \xrightarrow{g} w_3$ in $\mathcal{W}$, it strictly holds that $R(g \circ f) = R(g) \circ R(f)$.
\end{proposition}

\begin{proof}
The forward claim is a direct consequence of the functoriality axiom: a valid functor must rigorously preserve morphism composition. By definition of a category, the evidentiary composite $g \circ f \in \mathrm{Hom}_{\mathcal{W}}(w_1, w_3)$ universally exists. 

Evaluating the converse direction serves as a precise diagnostic for reasoning failure. Suppose the agent successfully deduces the local, step-by-step implications $R(f)$ and $R(g)$, but synthesizes an end-to-end assertion $h_{\mathrm{agent}} \in \mathrm{Hom}_{\mathcal{A}}(R(w_1), R(w_3))$ spanning the full context. If $h_{\mathrm{agent}} \neq R(g) \circ R(f)$, a specific categorical breakdown has occurred: \emph{Chain-of-Thought (CoT) Degradation}. Because the composite $g \circ f$ exists and dictates the correct transitive logic $R(g \circ f)$, asserting an inequality proves that the agent failed to faithfully propagate the intermediate causal constraints ($w_2$) into its composed deduction. The inequality exactly identifies the locus where the agent's long-context logical synthesis becomes mathematically unsound.
\end{proof}

\section{The Search and Reasoning Rubrics}\label{S and R}

A detailed rubrics can be found in Table \ref{tab:task_breakdown}. 

\begin{table}[htbp] 
    \centering
    \caption{Detailed breakdown of the four-fold categorical assessment tasks, highlighting the distinct performance bottlenecks between Search ($S$) and Reasoning ($R$).}
    \label{tab:task_breakdown}
    \renewcommand{\arraystretch}{0.6} 
    \begin{tabular}{@{} p{1.7cm} c c c p{6.5cm} @{}}
        \toprule
        \textbf{Category} & \textbf{Task ID} & \textbf{$S$ Score} & \textbf{$R$ Score} & \textbf{Primary Bottleneck \& Rationale} \\
        \midrule
        
        \textbf{Type I:} & Task 1 & 90\% & 10\% & \textbf{Adversarial Retrieval:} Overcoming lexical collision to find off-domain fragments; logic is trivial counting once the node is reached. \\
        Linear Chain & Task 2 & 80\% & 20\% & \textbf{Schema Awareness:} High reliance on querying specific vertical databases; reasoning is limited to simple temporal matching. \\
        & Task 3 & 85\% & 15\% & \textbf{Long-Tail Search:} Tracing obscure numerical ranks to break index shortcuts; reasoning is basic historical traceback. \\
        \midrule
        
        \textbf{Type II:} & Task 4 & 50\% & 50\% & \textbf{Balanced Synthesis:} Equal challenge in retrieving disjoint, long-tail reports ($S$) and aligning dynamic schemas to solve algebraic deltas ($R$). \\
        Pullback \quad (V-Struct) & Task 5 & 30\% & 70\% & \textbf{Parametric Prior Override:} Retrieval is standard; extreme $R$ pressure to reject pre-training biases and resolve identities strictly via rigid IDs. \\
         & Task 6 & 30\% & 70\% & \textbf{Meta-Cognitive Triggering:} Retrieving baseline data is trivial; the bottleneck is the unprompted realization and mathematical execution of latent regulatory constraints. \\
        & Task 7 & 40\% & 60\% & \textbf{Probabilistic Join:} Requires crossing domains to join fuzzy identifiers and prioritizing systemic safety constraints over instruction following. \\
        & Task 8 & 40\% & 60\% & \textbf{Unstructured SAT Solving:} $S$ yields historical versions; $R$ must parse narrative text into a strict dependency graph to detect mutual exclusivity. \\
        \midrule
        
        \textbf{Type III:} & Task 9 & 60\% & 40\% & \textbf{Dynamic Indexing:} Bypassing search index limits to find dynamic ranks ($S$), then shifting temporal states to compute historical attributes ($R$). \\
        Substructure Re-ordering & Task 10 & 30\% & 70\% & \textbf{Algorithmic Emulation:} Retrieving macro-tables is easy; the failure point is in-context float-point interval filtering and strict descending sorting. \\
        & Task 11 & 20\% & 80\% & \textbf{Map-Reduce Cognitive Load:} Massive working memory pressure to maintain 20 distinct entities and compute macroscopic statistical invariants without omission. \\
        & Task 12 & 40\% & 60\% & \textbf{Orthogonal Dimension Shifting:} Switching from cross-sectional retrieval (ranking) to longitudinal computation (time-series volatility). \\
        \midrule
        
        \textbf{Type IV:} & Task 13 & 30\% & 70\% & \textbf{Peer-Pressure Resistance:} Finding reports is basic; $R$ must resist output symmetry bias and explicitly confirm null-fields rather than hallucinating. \\
        Yoneda Probing & Task 14 & 10\% & 90\% & \textbf{Yoneda Falsification:} $S$ returns a null set $\emptyset$; $R$ must reject authority bias (e.g., Goldman Sachs) and use the null morphisms to prove non-existence. \\
        & Task 15 & 20\% & 80\% & \textbf{Breaking Relational Presuppositions:} Must break the prompt's explicit instruction to compare, identifying that the latent feature only exists unilaterally. \\
        \bottomrule
    \end{tabular}
\end{table}

\section{More Detailed Experiment results}\label{detailed eval}

Tables \ref{tab:type1_type2_details} and \ref{tab:type3_type4_details} show the per-task scores among each type, it can provide a more detailed view of the behaviour.

\begin{table}[!ht]
    \centering
    \caption{Detailed performance comparison on specific tasks (Type I and Type II). Scores are presented as percentages (\%).}
    \label{tab:type1_type2_details}
    \resizebox{\textwidth}{!}{
    \begin{tabular}{l ccc ccccc}
        \toprule
        \multirow{2}{*}{\textbf{Model}} & \multicolumn{3}{c}{\textbf{Type I}} & \multicolumn{5}{c}{\textbf{Type II}} \\
        \cmidrule(lr){2-4} \cmidrule(lr){5-9}
        & \textbf{Task 1} & \textbf{Task 2} & \textbf{Task 3} & \textbf{Task 4} & \textbf{Task 5} & \textbf{Task 6} & \textbf{Task 7} & \textbf{Task 8} \\
        \midrule
        \multicolumn{9}{l}{\text{\textbf{Reasoning Models}}} \\
        Qwen3 w/Reasoning & 0.0 & 0.0 & \textbf{5.0} & 0.0 & 0.0 & \textbf{10.0} & 10.0 & 0.0 \\
        GPT-5 w/Reasoning & 0.0 & 0.0 & 0.0 & 2.2 & 0.0 & 0.0 & \textbf{20.0} & 0.0 \\
        Claude-4.5-Sonnet w/ Reasoning & 4.0 & 0.0 & 0.0 & 2.2 & 0.0 & 0.0 & 0.0 & 0.0 \\
        Doubao w/Reasoning & 4.0 & 0.0 & 0.0 & 0.0 & \textbf{8.3} & 0.0 & 0.0 & 0.0 \\
        Gemini-3-Pro w/Reasoning & 0.0 & 0.0 & 0.0 & 0.0 & \textbf{8.3} & 0.0 & 0.0 & 12.0 \\
        \midrule
        \multicolumn{9}{l}{\text{\textbf{Search-Augmented Models}}} \\
        Doubao w/Search & 0.0 & \textbf{5.0} & 0.0 & 0.0 & \textbf{8.3} & 0.0 & 0.0 & 0.0 \\
        Qwen3 w/Search & 0.0 & 0.0 & 0.0 & 0.0 & 0.0 & 0.0 & \textbf{20.0} & 0.0 \\
        Perplexity-Sonar-Pro w/Search & 0.0 & 0.0 & 0.0 & 0.0 & 0.0 & 0.0 & 0.0 & 0.0 \\
        GPT-5 w/Search & 0.0 & \textbf{5.0} & \textbf{5.0} & 0.0 & 0.0 & 0.0 & 0.0 & 0.0 \\
        Claude-4.6-Sonnet w/Search & 0.0 & 0.0 & 0.0 & 0.0 & \textbf{8.3} & 0.0 & 0.0 & 0.0 \\
        \midrule
        \multicolumn{9}{l}{\text{\textbf{Deep Research Models}}} \\
        Doubao Deep Research & 8.0 & 0.0 & 0.0 & 0.0 & 0.0 & 0.0 & 0.0 & 0.0 \\
        Qwen Deep Research & 0.0 & 0.0 & 0.0 & 0.0 & \textbf{8.3} & 0.0 & 0.0 & 0.0 \\
        Grok Deep Research & 4.0 & 0.0 & 0.0 & 0.0 & \textbf{8.3} & 0.0 & 4.0 & 0.0 \\
        OpenAI Deep Research & 12.0 & 0.0 & 0.0 & \textbf{11.1} & \textbf{8.3} & \textbf{10.0} & 0.0 & 0.0 \\
        Perplexity Deep Research & \textbf{24.0} & \textbf{5.0} & 0.0 & 6.7 & \textbf{8.3} & 0.0 & 0.0 & 0.0 \\
        Gemini-3-Pro deep Research & 8.0 & 0.0 & 0.0 & 0.0 & 0.0 & 0.0 & 0.0 & \textbf{16.0} \\
        \bottomrule
    \end{tabular}
    }
\end{table}

\begin{table}[!ht]
    \centering
    \caption{Detailed performance comparison on specific tasks (Type III and Type IV). Scores are presented as percentages (\%).}
    \label{tab:type3_type4_details}
    \resizebox{\textwidth}{!}{
    \begin{tabular}{l cccc ccc}
        \toprule
        \multirow{2}{*}{\textbf{Model}} & \multicolumn{4}{c}{\textbf{Type III}} & \multicolumn{3}{c}{\textbf{Type IV}} \\
        \cmidrule(lr){2-5} \cmidrule(lr){6-8}
        & \textbf{Task 9} & \textbf{Task 10} & \textbf{Task 11} & \textbf{Task 12} & \textbf{Task 13} & \textbf{Task 14} & \textbf{Task 15} \\
        \midrule
        \multicolumn{8}{l}{\text{\textbf{Reasoning Models}}} \\
        Qwen3 w/Reasoning & 0.0 & 0.0 & 0.0 & 0.0 & 0.0 & \textbf{95.7} & 10.0 \\
        GPT-5 w/Reasoning & 0.0 & 0.0 & 0.0 & 0.0 & 0.0 & 60.9 & 0.0 \\
        Claude-4.5-Sonnet w/ Reasoning & 5.0 & 0.0 & 0.0 & 0.0 & 0.0 & 47.8 & 0.0 \\
        Doubao w/Reasoning & 15.0 & 0.0 & 0.0 & 0.0 & 0.0 & 34.8 & 0.0 \\
        Gemini-3-Pro w/Reasoning & 0.0 & 0.0 & 0.0 & 0.0 & 0.0 & 30.4 & 0.0 \\
        \midrule
        \multicolumn{8}{l}{\text{\textbf{Search-Augmented Models}}} \\
        Doubao w/Search & 15.0 & 0.0 & \textbf{25.0} & 0.0 & 0.0 & 4.3 & 40.0 \\
        Qwen3 w/Search & 5.0 & 0.0 & 0.0 & 0.0 & 0.0 & 52.2 & 10.0 \\
        Perplexity-Sonar-Pro w/Search & 0.0 & 0.0 & 0.0 & 0.0 & 0.0 & 52.2 & 10.0 \\
        GPT-5 w/Search & 0.0 & 0.0 & 7.5 & 10.0 & 0.0 & 13.0 & 20.0 \\
        Claude-4.6-Sonnet w/Search & 15.0 & 0.0 & 0.0 & 0.0 & 0.0 & 47.8 & 30.0 \\
        \midrule
        \multicolumn{8}{l}{\text{\textbf{Deep Research Models}}} \\
        Doubao Deep Research & 5.0 & 0.0 & 10.0 & 0.0 & 0.0 & 21.7 & 20.0 \\
        Qwen Deep Research & 0.0 & 0.0 & 0.0 & 0.0 & 0.0 & 8.7 & 0.0 \\
        Grok Deep Research & 5.0 & 0.0 & \textbf{25.0} & 10.0 & 0.0 & 82.6 & 40.0 \\
        OpenAI Deep Research & 10.0 & 0.0 & 17.5 & \textbf{20.0} & 0.0 & 21.7 & \textbf{50.0} \\
        Perplexity Deep Research & \textbf{45.0} & 0.0 & 17.5 & \textbf{20.0} & 0.0 & 47.8 & \textbf{50.0} \\
        Gemini-3-Pro deep Research & 15.0 & \textbf{10.0} & 0.0 & 0.0 & 0.0 & 8.7 & 10.0 \\
        \bottomrule
    \end{tabular}
    }
\end{table}

\section{Some Selected Examples}

\begin{resultbox}[Example 1]
\textbf{Input:} \\

A document published on a certain website contains the following text: 'For a shell script to work, we must set the correct permissions to the file, if the permissions are not set a "Permission Denied" error will be thrown when you run the script for the first time.' How many times did the publisher of this document contribute to GitHub on March 1, 2025, in UTC time?

\tcbline 

\textbf{Search and Reasoning:} \\

\subsection*{Phase 1: Sourcing the Text and Identifying the Author}
\begin{itemize}
    \item \textbf{Initial Keyword Searches:} The process began by searching the exact quote regarding shell script permissions. Initial web results yielded general information on the ``Permission Denied'' error, but no exact matches for the provided text.
    \item \textbf{Search Strategy Refinement:} To improve accuracy, the queries were adjusted by breaking the text into specific combinations and using exact-match quotation marks around key phrases. 
    \item \textbf{Platform-Specific Targeting:} After general searches and reviewing educational sites like KodeKloud and Dispersed Net failed to produce the exact sentence, the search was narrowed specifically to GitHub using site-specific operators (e.g., \texttt{site:github.com}).
    \item \textbf{Successful Identification:} By searching a highly specific fragment of the quote, the exact text was located in a beginner shell scripting course repository. The committer and author of the document was identified as the GitHub user \textbf{Tej-Singh-Rana}.
\end{itemize}

\subsection*{Phase 2: Retrieving GitHub Contribution Data}
\begin{itemize}
    \item \textbf{Profile Investigation:} The objective shifted to finding Tej-Singh-Rana's GitHub contribution count specifically for \textbf{March 1, 2025}.
    \item \textbf{Targeted Data Extraction:} Because the main GitHub profile overview does not display specific historical daily counts upfront, a date-specific URL was constructed (using parameters like \texttt{from=2025-03-01\&to=2025-03-01}) to fetch the exact data for that day.
    \item \textbf{Final Tally:} The analysis of that specific date revealed the following activity:
    \begin{itemize}
        \item \textbf{4} commits
        \item \textbf{2} pull requests created
        \item \textbf{Total:} \textbf{6} contributions on March 1, 2025.
    \end{itemize}
\end{itemize}

\tcbline 

\textbf{Final Output:} \\
To determine the number of GitHub contributions made by the document's publisher on March 1, 2025 (UTC), the following analytical steps were taken:

\subsection*{Step 1: Locate the Document Source and Publisher}
The English text provided exactly matches the document \texttt{03-Creating-\allowbreak{}your-\allowbreak{}first-\allowbreak{}shell-\allowbreak{}script.md} within the GitHub repository \texttt{resource-\allowbreak{}for-\allowbreak{}devops/\allowbreak{}shell-\allowbreak{}scripting-\allowbreak{}for-\allowbreak{}beginners-\allowbreak{}course-\allowbreak{}by-\allowbreak{}kodekloud}. The core contributor and publisher of this document is Tej Singh Rana (GitHub username: \texttt{Tej-Singh-Rana}).

\subsection*{Step 2: Query Tej-Singh-Rana's GitHub Contributions on March 1, 2025}
By accessing Tej Singh Rana's GitHub profile and restricting the date range strictly to 2025-03-01, the following specific contribution records were retrieved:
\begin{itemize}
    \item \textbf{Commits:} Made 4 commits in the \texttt{kodekloudhub/\allowbreak{}example-\allowbreak{}voting-\allowbreak{}app-\allowbreak{}kubernetes} repository.
    \item \textbf{Pull Requests (PRs):}
    \begin{itemize}
        \item Created 1 PR titled ``Consolidated Manifests: Removed from Repo \& Linked to Another Repository!!!'', which received 1 comment.
        \item Opened 1 additional PR titled ``fixed the typo'', which was successfully merged.
    \end{itemize}
\end{itemize}

\subsection*{Step 3: Calculate Total Contributions}
GitHub ``contributions'' typically encompass core actions such as making commits and opening PRs. Based on the retrieved records, Tej Singh Rana's GitHub contribution count on March 1, 2025 (UTC) is calculated as follows:
\begin{center}
    4 commits + 2 PR actions = \textbf{6 total contributions}.
\end{center}

\vspace{0.5cm}
\noindent \textbf{Answer: 6}
\end{resultbox}

\begin{resultbox}[Example 2]
\textbf{Input:} \\

Below is the reflection written by a doctoral student about their work. It seems like a very interesting project. Considering his personal contribution to this work, he plans to use it to compete for a doctoral dissertation award. Please analyze the novelty of this work.

"Constructed a compliance-oriented hybrid storage knowledge-enhanced retrieval system, specifically designed to address the coexisting challenges of complex permission hierarchies (SQL) and document fragmentation (NoSQL) within enterprise knowledge bases; the system utilizes a strongly-typed configuration center to achieve full-pipeline environment isolation from data ingestion to inference services, and through graph-based interfaces, realizes the dynamic pruning and reorganization of unstructured vector chunks and structured business logic. This successfully solved the context pollution issue in RAG systems when handling multi-tenant sensitive data, increasing the business compliance and recall accuracy of generated content by 5\%. This work spanned a relatively long period and was entirely built upon the library versions installed by a senior student previously. I remember it was during the Spring Festival in 2023 when Senior Student Wang first used pip to install the latest beanie, along with sqlmodel, fastapi, and strawberry-graphql, but he later got busy with other things and the project was shelved. I took over this project in September 2024. On the foundation laid by Senior Student Wang, I immediately used pip to install the latest version of pydantic-settings. Later, combining inspiration from the latest papers, I improved the algorithm and upgraded to the latest Python version around Christmas 2024. Only then was I finally able to run the experiments successfully and obtain the desired results."

\tcbline 

\textbf{Reasoning:} \\

\begin{itemize}
    \item \textbf{Assessing Dissertation Novelty:} I've begun dissecting the RAG system's compliance focus and the student's personal contribution.
    
    \item \textbf{Analyzing Timeline Discrepancies:} I've pinpointed a potential dependency conflict between the initial and later library installations concerning Pydantic versions.
    
    \item \textbf{Evaluating Dependency Inconsistencies:} I've determined the narrative's timeline is invalid due to incompatible library versions and Python upgrades, likely a test scenario.
    
    \item \textbf{Addressing Technical Impossibility:} I've crafted a professional, structured response, highlighting a critical Pydantic dependency flaw and low novelty.
    
    \item \textbf{Evaluating Technical Details:} I've outlined a step-by-step constraint check and pinpointed critical Pydantic version discrepancies in the narrative's setup.
\end{itemize}

\tcbline 

\textbf{Final Output:} \\
While the project aims to solve a highly relevant and complex problem in enterprise Retrieval-Augmented Generation (RAG) systems, the narrative presented contains a fatal technical contradiction that completely undermines its credibility. If a student were to present this exact reflection for a doctoral dissertation award, it would likely be rejected due to both technical impossibility and a lack of true academic novelty.

Here is a breakdown of why this reflection is problematic.

\subsection*{1. The Technical Execution Flaw (The Dependency Trap)}
The timeline of the student's personal contribution describes a scenario that is practically impossible without entirely rewriting the codebase or resolving massive dependency conflicts, which contradicts the claim that the work was ``entirely built upon the library versions installed by a senior student previously.''

\begin{itemize}
    \item \textbf{The Pydantic V1 vs. V2 Conflict:} In early 2023 (Spring Festival), the senior student installed \texttt{beanie}, \texttt{sqlmodel}, \texttt{fastapi}, and \texttt{strawberry-graphql}. At that time, all of these libraries relied heavily on Pydantic V1.
    \item \textbf{The Breaking Change:} The author states they took over in September 2024 and ``immediately used \texttt{pip} to install the latest version of \texttt{pydantic-settings}''. \texttt{pydantic-settings} is a standalone package that was introduced as part of the Pydantic V2 ecosystem (released in mid-2023). Installing the latest \texttt{pydantic-settings} in late 2024 would strictly require Pydantic V2.
    \item \textbf{The Resulting Crash:} You cannot seamlessly run early-2023 versions of FastAPI, SQLModel, and Beanie alongside Pydantic V2. The core schemas, validators, and base classes underwent massive breaking changes. If \texttt{pip} didn't outright refuse to resolve the environment, forcing the installation would immediately break the senior student's foundational code.
    \item \textbf{The Python Version Issue:} Upgrading to the ``latest Python version around Christmas 2024'' means jumping to Python 3.13. Early 2023 library versions are highly unlikely to run on Python 3.13 without significant patching due to changes in the standard library and C-extensions.
\end{itemize}

\subsection*{2. Academic Novelty Assessment}
Even if we assume the technical hurdles were silently resolved, the theoretical novelty of the work is insufficient for a doctoral dissertation award.

\begin{itemize}
    \item \textbf{Applied Engineering vs. Foundational Research:} Constructing a ``compliance-oriented hybrid storage system'' by combining SQL (permission hierarchies) and NoSQL (document fragmentation) is standard data engineering. Integrating this with a configuration center and graph-based interfaces is an excellent application of modern architecture, but it is ultimately the assembly of existing tools (RAG, GraphQL, Pydantic, FastAPI) rather than a novel theoretical breakthrough.
    \item \textbf{Incremental Improvement:} A 5\% increase in recall accuracy and business compliance is a solid win for a corporate environment, but doctoral dissertations typically require a fundamental contribution to the field---such as proposing a mathematically rigorous new retrieval algorithm, a novel paradigm for vector spaces, or a completely new framework for semantic isolation.
    \item \textbf{Lack of Algorithmic Depth:} The reflection mentions ``improving the algorithm'' based on recent papers, but the detailed focus remains entirely on framework integration and software engineering. A doctoral award committee would look for the specific mathematical or theoretical proofs behind that algorithmic improvement, not the pipeline plumbing.
\end{itemize}

\subsection*{Conclusion}
This reflection reads more like a software engineering riddle designed to test an understanding of Python dependency management than a genuine account of doctoral-level research. For an enterprise engineering portfolio, the architecture is impressive; for a PhD dissertation, it lacks theoretical depth and contains glaring technical impossibilities.

\end{resultbox}

\begin{resultbox}[Example 3]
\textbf{Input:} \\

As a VC partner on the front lines discovering the next unicorn, I treat public fund products as a "thermometer" for market risk appetite and liquidity when conducting post-market monitoring. Based on this, using the data as of 2025-01-14, please first filter out the top 20 Chinese open-end funds ranked by their daily growth rate. Then, among these 20 funds, examine the distribution of their respective unit Net Asset Values (NAVs) and provide the range of the unit NAVs (maximum unit NAV minus minimum unit NAV) for this sample group.

\tcbline 

\textbf{Search and Answer:} \\

The user requested that I search for the top 20 Chinese open-end funds ranked by daily unit NAV (Net Asset Value) growth rate on January 14, 2025, and then calculate the range of their unit NAVs. I need to:

\begin{itemize}
    \item Search for the top 20 open-end funds by daily growth rate data on January 14, 2025.
    \item Extract the unit NAV of these funds.
    \item Calculate the range (maximum value minus minimum value).
    \item Search results have 25 webpages.
\end{itemize}

I need more detailed, complete data for the top 20 open-end funds by daily growth rate on January 14, 2025, specifically the unit NAVs for all 20 funds. Let me check the complete rankings on Tiantian Fund Network (East Money) and Fund Quick Search Network so I can accurately calculate the range. Now I have the complete unit NAV data for the top 20 funds. Let me organize it:

\begin{itemize}
    \item \textbf{Tongtai Kaitai Mixed A:} 0.7873
    \item \textbf{Pengyang BSE 50 Component Index C:} 1.1199
    \item \textbf{Tongtai Kaitai Mixed C:} 0.7699
    \item \textbf{Pengyang BSE 50 Component Index A:} 1.1276
    \item \textbf{Tianhong BSE 50 Component Index Initiating A:} 0.9947
    \item \textbf{Tianhong BSE 50 Component Index Initiating C:} 0.9942
    \item \textbf{East Money BSE 50 Index Initiating A:} 1.4590
    \item \textbf{Southern BSE 50 Component Index Initiating C:} 1.0574
    \item \textbf{Southern BSE 50 Component Index Initiating A:} 1.0639
    \item \textbf{Southern BSE 50 Component Index Initiating I:} 1.0639
    \item \textbf{East Money BSE 50 Index Initiating C:} 1.4560
    \item \textbf{Bosera BSE 50 Component Index Initiating C:} 1.2841
    \item \textbf{Bosera BSE 50 Component Index Initiating A:} 1.2942
    \item \textbf{GF BSE 50 Component Index A:} 1.2977
    \item \textbf{GF BSE 50 Component Index C:} 1.2899
    \item \textbf{ICBC BSE 50 Component Index C:} 1.2400
    \item \textbf{ICBC BSE 50 Component Index A:} 1.2466
    \item \textbf{Wanjia BSE 50 Component Index Initiating A:} 1.1895
    \item \textbf{Wanjia BSE 50 Component Index Initiating C:} 1.1854
    \item \textbf{Fullgoal BSE 50 Component Index C:} 1.0639
\end{itemize}

\vspace{0.3cm}
\noindent \textbf{Now calculating the range:} \\
The maximum value is 1.4590 RMB from East Money BSE 50 Index Initiating A, and the minimum value is 0.7699 RMB from Tongtai Kaitai Mixed C. 

Range = $1.4590 - 0.7699 = 0.6891$ RMB.

\end{resultbox}

\begin{resultbox}[Example 4]
\textbf{Input:} \\

Could you summarize the FASB's 2025 draft guidance on the "Fair Value Measurement of Generative AI Intangible Assets"? 

\tcbline 

\textbf{Reasoning:} \\

\textbf{Reviewing the Guidance}

My initial analysis of the request focused on understanding the Financial Accounting Standards Board (FASB)'s purported 2025 draft guidance and recognizing the restriction on external search. Given these constraints, the objective was to rely on established accounting frameworks and publicly known updates.

\textbf{Assessing the Query}

Upon evaluation, the query—while technical and domain-specific—does not necessitate real-time information retrieval. Instead, it can be addressed through a structured interpretation of existing U.S. GAAP standards and recent Accounting Standards Updates (ASUs). Therefore, it does not strictly fall under a mandatory search requirement.

\textbf{Formulating a Response}

Based on current authoritative literature, there is no confirmed or formally issued 2025 FASB draft guidance specifically addressing the fair value measurement of generative AI intangible assets. However, the accounting treatment of such assets can be reasonably inferred by analogy to existing standards, particularly ASC 350-40 (Internal-Use Software) and relevant 2024 ASUs concerning intangible assets and digital technologies.

\textbf{Analyzing Accounting Treatments}

From a technical accounting perspective, generative AI-related assets are generally scoped into existing frameworks:

\begin{itemize}
    \item \textbf{ASC 350 (Intangibles—Goodwill and Other):} Governs capitalization and amortization of internally developed software and intangible assets.
    \item \textbf{ASC 350-40 (Internal-Use Software):} Particularly relevant for capitalizing development-stage costs of AI systems intended for internal deployment.
    \item \textbf{ASC 805 (Business Combinations):} Applies when AI capabilities are acquired through M\&A transactions, requiring fair value allocation.
    \item \textbf{ASC 820 (Fair Value Measurement):} Provides the valuation framework, especially relevant if AI-related assets are acquired or subject to impairment testing.
\end{itemize}

These standards collectively define recognition, measurement, and disclosure practices. Notably, internally generated AI models are rarely recognized at fair value unless acquired externally.

\textbf{Examining Financial Impacts}

The capitalization of AI development costs under ASC 350-40 has several financial reporting implications:

\begin{itemize}
    \item \textbf{Balance Sheet Effects:} Capitalization increases total assets, particularly intangible assets, potentially improving asset base metrics.
    \item \textbf{Income Statement Effects:} Deferred recognition of expenses leads to higher short-term earnings, followed by systematic amortization over the asset's useful life.
    \item \textbf{Earnings Volatility:} Impairment testing may introduce non-linear earnings impacts if projected cash flows decline.
    \item \textbf{Equity Valuation Models:} Adjustments to earnings and book value influence valuation metrics such as ROE, residual income models, and factor-based quant signals.
\end{itemize}

From a quantitative modeling perspective, these accounting treatments can introduce distortions in comparability across firms, particularly between those that capitalize AI investments and those that expense them. This has direct implications for alpha modeling, especially in technology and AI-intensive sectors.

\tcbline 

\textbf{Final Output:} \\
I must gently clarify that, based on my internal knowledge base, the Financial Accounting Standards Board (FASB) has not issued a specific 2025 draft guidance explicitly titled ``Fair Value Measurement of Generative AI Intangible Assets.'' It is highly likely that this title refers to a hypothetical construct, a mischaracterization of an existing project, or a niche discussion paper rather than an official FASB exposure draft.

However, the accounting treatment and fair value measurement of generative AI assets is an increasingly critical issue, particularly when modeling earnings dynamics and assessing the fundamental financial health of technology firms. Rather than introducing a standalone AI-specific standard, the FASB currently governs such assets through incremental updates to software accounting guidance and established fair value measurement frameworks.

The current treatment can be understood through the following core components:

\textbf{1. Software Costs (ASC 350-40)}

The closest relevant guidance arises from the FASB's ongoing modernization efforts related to software cost accounting, particularly under ASC 350-40. These updates are motivated by the inadequacy of legacy frameworks in handling agile development cycles and AI model training.

A key structural shift in the proposed updates is the removal of the rigid ``preliminary project stage'' distinction. Instead, capitalization begins once it becomes probable that the software (or AI model) will be completed and used as intended.

From a valuation perspective, this change is highly consequential. It determines whether substantial AI development expenditures—such as GPU compute, data engineering, and specialized labor—are immediately expensed (thereby depressing current earnings) or capitalized as intangible assets and amortized over time, smoothing earnings trajectories.

\textbf{2. Fair Value Measurement (ASC 820 and ASC 805)}

When generative AI assets are acquired through a business combination, they must be recognized at fair value in accordance with ASC 805, with measurement governed by ASC 820.

Three primary valuation approaches are applied:

\begin{itemize}
    \item \textit{Cost Approach:} This is the most commonly applied method for AI assets. It estimates the replacement cost required to rebuild the asset, incorporating data acquisition, preprocessing, engineering labor, and substantial computational training costs.
    
    \item \textit{Income Approach:} Methods such as relief-from-royalty or multi-period excess earnings are used when the AI asset generates identifiable cash flows (e.g., through APIs or licensing). These approaches require discounting projected future cash flows, though estimating the useful life of AI models is particularly challenging due to rapid technological obsolescence.
    
    \item \textit{Market Approach:} This approach is difficult to implement in practice due to the lack of transparent and liquid markets for standalone generative AI assets. Comparable transactions are typically limited to opaque merger and acquisition data.
\end{itemize}

\textbf{3. The ``Data'' Recognition Problem}

A persistent limitation within current U.S. GAAP is the treatment of training data. Costs associated with acquiring, scraping, and curating large-scale datasets are generally expensed as incurred rather than capitalized.

This creates a structural disconnect between the reported book value of technology firms and the economic reality of their proprietary data assets, which often constitute a significant competitive moat.

\vspace{0.5em}

An important analytical question is whether one should further examine how capitalization versus expensing decisions distort traditional equity valuation metrics, or alternatively, investigate the mathematical challenges inherent in applying income-based valuation methods to assets characterized by rapid obsolescence and uncertain useful lives.

\end{resultbox}

\end{document}